%% file: main.tex
\documentclass{article}
\usepackage{arxiv}
\usepackage{natbib}
\setcitestyle{round}

\usepackage[utf8]{inputenc} 
\usepackage[T1]{fontenc}    
\usepackage{lmodern}
\usepackage{parskip}
\usepackage[hypertexnames=false]{hyperref}  

\input{sections/0-Preamble.tex}

\usepackage[utf8]{inputenc} %
\usepackage[T1]{fontenc}    %
\usepackage{hyperref}       %
\usepackage{url}            %
\usepackage{booktabs}       %
\usepackage{amsfonts}       %
\usepackage{nicefrac}       %
\usepackage{microtype}      %
\usepackage{xcolor}         %

\usepackage{xspace}
\newcommand{\ours}{\textbf{MaD-Mix}\xspace}
\newcommand{\K}{K_\mathrm{MM}\xspace} 
\newcommand{\Ks}{K\xspace}

\newcommand{\h}{\alpha\xspace}
\newcommand{\hs}{\h'\xspace}
\newcommand{\myalpha}{p\xspace}

\input{def}

\usepackage{graphicx}
\usepackage{subcaption}
\usepackage[table,xcdraw]{xcolor}
\usepackage{tabularx}
\newcolumntype{C}{>{\centering\arraybackslash}X}
\usepackage{wrapfig} %

\usepackage{mathtools}
\mathtoolsset{showonlyrefs=true}

\title{MaD-Mix: Multi-Modal Data Mixtures  via Latent Space Coupling for Vision-Language Model Training}

\date{~}

\author{%
Wanyun Xie\\
  LIONS, EPFL\\
\texttt{wanyun.xie@epfl.ch} \\  
   \And
    Francesco Tonin \\
  LIONS, EPFL\\
\texttt{francesco.tonin@epfl.ch} 
   \And
    Volkan Cevher \\
  LIONS, EPFL\\
\texttt{volkan.cevher@epfl.ch} 
}

\definecolor{myorange}{HTML}{FF8C00} %
\definecolor{myblue}{HTML}{6495ED}   %
\definecolor{mygrey}{HTML}{9c9ca1}
\definecolor{mypurple}{HTML}{964a8b}
\definecolor{myred}{HTML}{E42536} %

\begin{document}

\maketitle

\begin{abstract}

Vision-Language Models (VLMs) are typically trained on a diverse set of multi-modal domains, yet current practices rely on costly manual tuning.
We propose \ours, a principled and computationally efficient framework that derives multi-modal data mixtures for VLM training.
\ours formulates data mixing as modality-aware domain alignment maximization and obtains closed-form multi-modal alignment scores from the Fenchel dual through inter-modal coupling variables.
\ours systematically handles domains with missing modalities, allowing for the integration of language-only domains.
Empirical evaluations across 0.5B and 7B models demonstrate that \ours \textit{accelerates} VLM training across diverse benchmarks.
\ours matches human-tuned data mixtures using 22\% fewer training steps in image-text instruction tuning.
In complex tri-modal video-image-text scenarios, where manual tuning becomes impractical, \ours boosts average accuracy over uniform weights, with negligible mixture computation overhead (\textless 1 GPU-hour), enabling scalable mixture design for modern VLM pipelines.\looseness=-1

\end{abstract}

\let\thefootnote\relax\footnotetext{Correspondence to: Wanyun Xie <wanyun.xie@epfl.ch>, Francesco Tonin <francesco.tonin@epfl.ch>.}

\section{Introduction} \label{sec:intro}
\input{sections/1-Introduction}

\section{Related Works} \label{sec:related}
\input{sections/2-RelatedWoks}

\section{Multi-modal Mixing with Modality-aware Domain Alignment} \label{sec:method}
\input{sections/3-Method}

\section{Experiments} \label{sec:experiments}
\input{sections/8-Experiments}

\section{Conclusion} \label{sec:conclusion}
\input{sections/9-Conclusion}

\section*{Acknowledgements}
This work was supported by the Swiss National Science Foundation (SNSF) under grant numbers 2000-1-240094 and 200021\_205011.
This work was supported with project ID \#37 as part
of the Swiss AI Initiative, through a grant from the ETH
Domain and computational resources provided by the
Swiss National Supercomputing Centre (CSCS) under
the Alps infrastructure.
This work was supported by Hasler Foundation Program: Hasler Responsible AI (project number 21043).
Research was sponsored by the Army Research Office and was accomplished under Grant Number W911NF-24-1-0048.

\bibliographystyle{unsrtnat}
\bibliography{references}

\clearpage
\appendix
\input{sections/10-Appendix}

\end{document}

%% file: sections/0-Preamble.tex
\usepackage{microtype}
\usepackage{amsmath,graphicx}
\usepackage{amssymb,amsfonts,amsthm}
\usepackage{mathrsfs}
\usepackage{mathtools}
\usepackage[mathscr]{eucal}
\usepackage{bbm}
\usepackage{bm}
\usepackage{fontawesome}

\usepackage[dvipsnames]{xcolor}

\usepackage{url}
\usepackage{hyperref}
\usepackage[capitalize,noabbrev]{cleveref}
\hypersetup{
    colorlinks,
    linkcolor={red!75!black},
    citecolor={blue!80!black},
    urlcolor={blue!80!black}
}
\usepackage{booktabs}

\usepackage{multicol}
\usepackage{multirow}
\usepackage{paralist, enumitem}

\usepackage{algorithm,algpseudocode}
\algnewcommand\algorithmicinput{\textbf{Input:}}
\algnewcommand\Input{\item[\algorithmicinput]}
\algnewcommand\algorithmicoutput{\textbf{Output:}}
\algnewcommand\Output{\item[\algorithmicoutput]}

\usepackage{thmtools, thm-restate}
\theoremstyle{plain}
\newtheorem{theorem}{Theorem}[section]
\newtheorem{proposition}[theorem]{Proposition}
\newtheorem{lemma}[theorem]{Lemma}
\newtheorem{corollary}[theorem]{Corollary}
\theoremstyle{definition}

\theoremstyle{remark}
\newtheorem{remark}[theorem]{Remark}
\Crefname{proposition}{Proposition}{Propositions}
\crefname{problem}{Problem}{Problems}
\Crefname{lemma}{Lemma}{Lemmas}

\usepackage{aliascnt}
\newaliascnt{problem}{equation}
\aliascntresetthe{problem}
\creflabelformat{problem}{#2\textup{#1}#3}

\makeatletter

\def\endproblem{\eqno \hbox{\@eqnnum}$$\@ignoretrue}
\makeatother

\crefname{problem}{Problem}{Problems}
\crefname{algorithm}{Algorithm}{Algorithms}
\crefname{figure}{Figure}{Figures}
\crefname{proposition}{Proposition}{Propositions}
\crefname{appendix}{Appendix}{Appendix}

\usepackage{array}
\newcolumntype{H}{>{\setbox0=\hbox\bgroup}c<{\egroup}@{}} %

\newcommand{\eg}{{\em e.g.~}}

\newcommand{\cD}{\mathcal{D}}

\newcommand{\cL}{\mathcal{L}}

\newcommand{\cO}{\mathcal{O}}

\newcommand\R{\mathbb{R}}

\newcommand{\scalone}[1]{\langle \rangle}

\newcommand{\norm}[1]{\left\lVert {#1} \right\rVert}

\def\fro{\mathrm{F}}

%% file: def.tex
\def\dag{{\cal y}}
\def\R{\mathbb{R}}

\def\beq{\begin{equation}}
\def\eeq{\end{equation}}
\def\beqa{\begin{eqnarray}}
\def\eeqa{\end{eqnarray}}
\def\balign{\begin{align}}
\def\ealign{\end{align}}
\def\bpr{\begin{proof}}
\def\epr{\end{proof}}
\def\bth{\begin{theorem}}
\def\eth{\end{theorem}}
\def\blm{\begin{lemma}}
\def\elm{\end{lemma}}
\def\bprop{\begin{proposition}}
\def\eprop{\end{proposition}}
\def\bcr{\begin{corollary}}
\def\ecr{\end{corollary}}

\def\eg{{\it e.g.,\ \/}}

\def\and {{\rm and}}

\def \R{\mathbb{R}}

%% file: sections/1-Introduction.tex
Vision-Language Models (VLMs) have advanced significantly with the availability of large-scale multi-modal datasets. 
The training data for VLMs is typically a complex mixture from numerous domains and multiple modalities \citep{bai2023qwenvlversatilevisionlanguagemodel,liu2023visualinstructiontuning,li2024llavaonevisioneasyvisualtask,liu2024improvedbaselinesvisualinstruction}. 
For example, LLaVA-OneVision is trained on 20.6\% Doc/Chart/Screen, 20.1\% Math/Reasoning, and 8.9\% OCR data, etc., and includes text and vision modalities \citep{li2024llavaonevisioneasyvisualtask}. 
Since such domains help maintain and balance the skill distribution that a trained large multi-modal model should cover \citep{li2024llavaonevisioneasyvisualtask}, many studies follow the topic or capability-oriented rule with domain structure when collecting data, such as LLaVA \citep{liu2023visualinstructiontuning,li2024llavaonevisioneasyvisualtask}, Qwen \citep{bai2023qwen_report, yang2025qwen3}, LLAMA \citep{dubey2024llama}, Gemini \citep{team2023gemini}, InstructBLIP \citep{dai2023instructblip}, and others \citep{li2025otter, tong2024cambrian, chen2024expanding, laurenccon2024matters}.
Moreover, the composition of these domains critically impacts VLM effectiveness \citep{bai2023qwenvlversatilevisionlanguagemodel,li2024llavaonevisioneasyvisualtask,liu2024llavanext, gadre2023datacompsearchgenerationmultimodal}.
\textit{``How to systematically determine the proportions of each domain for VLM training without costly tuning?''}
is an essential question and remains an open challenge.

Existing strategies for constructing multi-modal data mixtures often lack a formal methodology.
Data recipes for many state-of-the-art models are not publicly released,
while open-source models typically rely on expensive manual tuning or heuristic adjustments based on developers' experience \citep{bai2023qwenvlversatilevisionlanguagemodel,li2024llavaonevisioneasyvisualtask}. 
For instance, Flamingo relies on empirically-tuned mixtures \citep{alayrac2022flamingo}, LLaVA-NeXT manually adds data domains to improve specific skills \citep{liu2024llavanext}, and InstructBLIP uses a simple sampling heuristic to handle data imbalance \citep{dai2023instructblip}.
Such approaches are inefficient and unscalable, especially as datasets grow.
Consequently, a principled and efficient methodology for optimizing the data mixture for VLMs is notably absent.

Although data mixing strategies have shown considerable success in Large Language Model (LLM) training \citep{xie_doremi_2023,fan_doge_2024,liu2024regmix,kang2024autoscale}, directly transferring these unimodal approaches to VLMs presents significant challenges due to their fundamental differences. 
The VLM data mixing problem introduces two unique challenges: \textit{(i)} integrating features from \textbf{different modalities} (\eg text and vision); and \textit{(ii)} handling domains with \textbf{missing modalities}, which frequently arises in VLM training where some domains include text-image paired data for visual learning, while others have text-only data for preserving linguistic abilities. 
Therefore, a specialized, modality-aware methodology is required for effective VLM data mixing.

In this paper, we introduce \ours, a framework for { systematically determining domain weights} in VLM training.
\ours computes modality-aware scores by casting multi-modal data mixing as a latent-space coupling objective and deriving the scores via the corresponding dual solution.
We achieve cross-modal integration via shared latent variables that map multi-modal features into a common space. 
In addition, \ours handles missing modalities by { explicitly decoupling them from the optimization process}, ensuring they do not propagate noise into the alignment objective.
The resulting scores directly translate into resampling weights, yielding 
higher VLM training efficiency and {mitigating the cost of manual mixture tuning.}

Specifically, the novelty and contribution of this work can be summarized as:

\begin{itemize}
    \item 
    We propose \ours, a principled fusion-based framework for reweighting multi-modal domains in VLM training. 
    At the core of \ours is a robust fusion mechanism that derives domain weights by measuring the projection of each domain onto a shared latent space. Specifically, we show that coupling modalities via dual variables with an alignment objective leads to a weighting scheme acting as spectral soft-thresholding. We show that this formulation yields a computationally efficient closed-form solution.
    
    \item We address the challenge of heterogeneous multi-modal data integration. Our method is explicitly designed to handle domains with differing modalities (e.g., mixing image-text paired and text-only domains) by decoupling missing modalities from the objective and ensuring no noise is introduced by incomplete data.
    
    \item  We empirically validate \ours on 0.5B and 7B VLMs, demonstrating its effectiveness and negligible cost. 
    On the 0.5B model, it matches expert-tuned performance
     with a 1.28$\times$ speedup
     in image-text tuning.
    In complex video-image-text scenarios, it outperforms uniform weighting using only 33\% of the training steps.
    Notably, the computational cost of computing domain weights is less than 1 GPU-hour.
\end{itemize}

%% file: sections/2-RelatedWoks.tex
\paragraph{Data composition in VLMs.}
The performance of modern VLMs is critically dependent on the composition of their training data.
A standard practice in the field is to curate data into distinct, skill-oriented domains to ensure a balanced set of capabilities. 
For example, the development of the LLaVA family \citep{liu2023visualinstructiontuning,li2024llavaonevisioneasyvisualtask,liu2024improvedbaselinesvisualinstruction} involved explicitly adding new data domains like DocVQA and ChartQA to improve targeted skills such as OCR and chart understanding. They openly release the LLaVA-OneVision~\citep{li2024llavaonevisioneasyvisualtask} datasets as collections of domain-specific data, which we use in our experiments.
Similarly, the Qwen-VL~\citep{bai2023qwenvlversatilevisionlanguagemodel} and Gemini \citep{team2023gemini} employ a multi-stage training pipeline that combines multi-modal data with text-only dialogue to maintain language capabilities.
InstructBLIP~\citep{gu2025infinitymmscalingmultimodalperformance} also groups 26 public datasets into 11 categories to cover a wide variety of tasks
and capabilities.
Many other works \citep{li2025otter, tong2024cambrian, chen2024expanding, laurenccon2024matters} follow such a capability-oriented rule to construct domains.
While preliminary steps in the data pipeline such as data cleaning, toxicity removal, quality filtering, and coreset selection are also important aspects, our work focuses on the subsequent challenge of weighting the given pre-curated, skill-specific domains.\looseness=-1

\paragraph{Data mixing.}
Despite the widespread practice of domain-structured data curation in VLMs, the subsequent step of determining the proportional mixture of these domains largely relies on developer intuition or costly empirical tuning.
For instance, LLaVA-One~\citep{li2024llavaonevisioneasyvisualtask} and Flamingo~\citep{alayrac2022flamingo} manually tuned domain weights for their promising performance.
Other approaches, like that for LLaVA-NeXT~\citep{liu2024llavanext}, involve reactively adding new data to address perceived skill gaps, which is inefficient and heuristic.
InstructBLIP~\citep{gu2025infinitymmscalingmultimodalperformance} observes that ignoring the mixing problem in VLMs leads to unstable training and harms performance. 
While data mixing has been studied more formally for unimodal LLMs, these approaches are fundamentally ill-suited for VLMs. 
Most of them \citep{xie_doremi_2023, fan_doge_2024, liu2024regmix, ye2024data,kang2024autoscale} rely on proxy models' training, which is difficult to combine in the multi-stage VLM pipeline.
Recent directions \citep{xie2025chameleon,zhang2025domainvec} integrate into LLM training, but they are not designed to handle multimodal features and cannot manage domains with missing modalities.

%% file: sections/3-Method.tex
We propose \ours, a fusion-based framework for principled multi-modal data mixing.
VLM training data presents unique challenges: feature heterogeneity and varying modality availability across domains.
We address these by formulating mixing as
measuring the projection of each domain onto a consensus direction.
\cref{subsec:alignment_score} derives the multi-modal objective by coupling domain contributions in a shared latent space. In \cref{subsec:miss_modality}, we extend this formulation to handle missing modalities explicitly. 
The practical mixing pipeline of \ours is shown in \Cref{fig:pipeline}.

\input{fig/pipeline.tex}

\paragraph{Setup and objective.} 
Let \( \cD_\text{MM} = \{D_i\}_{i\in[k]} \) be the set of $k$ VLM training data domains (e.g., Math, OCR, etc.), where $[k]$ denotes the set of integers $\{1,\ldots,k \}$.
These domains define the target skill sets that the final trained VLM should possess.
While samples within a domain $D_i$ share the same modalities, the available modalities may vary across domains.
Each sample $a^{[v]}$ from modality $v, v\in[V]$ (e.g., vision or text) can be represented through its semantic embedding $h^{(L)}(a^{[v]})$ extracted from the $L$-th hidden layer of the pretrained VLM $h$.
The domain embedding $x_i^{[v]} \in \R^d$ for the $v$-th modality can be constructed as the semantic centroid $x_i^{[v]} = \frac{1}{|D_i|} \sum_{a^{[v]} \in D_i} h^{(L)}(a^{[v]})$, which can effectively represent data domains thanks to the high-dimensional, non-linear representations learned by Transformers~\citep{xie2025chameleon,ling2025}.
The data mixing objective is to determine a domain weight vector \( \myalpha \in \Delta_k \)
\citep{albalak2023efficient,fan_doge_2024} for VLM training, where $p_i$ serves as the sampling probability for domain $D_i$.\looseness=-1

\subsection{Multi-modal domain alignment scores} \label{subsec:alignment_score}
We first consider the single-modality setting.
Given multiple domains representing various capabilities, the goal of VLM training is to learn generalizable representations that can transfer effectively across tasks.
We seek a projection direction $w$ in embedding space that captures the shared structure underlying all $k$ domains. We define the alignment score $S'_i$ of domain $D_i$ as its projection onto this direction. 
This leads to the following primal optimization problem:
\begin{equation} \label{eq:primal}
\min_{w,e} \frac{1}{2\lambda} \sum_{i=1}^k e_i^2 + \frac{1}{2} \norm{w}_\fro^2 \; \;
\text{s.t. } e_i = 1 - w^\top x_i, \, i \in [k],
\end{equation}
where $w \in \mathbb{R}^d$ represents the projection vector, $e = [e_1, \ldots, e_k] \in \mathbb{R}^k$ denotes the individual projection errors for each domain, and $\lambda > 0$ is a regularization parameter.
By assigning a uniform target value of $1$ for all domains, \eqref{eq:primal} drives the optimization to identify a consensus direction $w$ that aligns with the entire collection of domain embeddings $\{x_i\}_{i \in [k]}$, rather than being skewed toward any specific one.

\hfill
\begin{remark}[Beamforming]
The form \eqref{eq:primal} has a well-defined interpretation from signal processing.
It is analogous to a \emph{beamformer}~\citep{van2002optimum} where the mean embedding acts as the desired steering vector
and
the covariance matrix represents the domain dispersion.
The optimal projection corresponds to directions that balance maximizing the scores with the shared signal while minimizing interference through the covariance $(\Sigma + \lambda I_d)^{-1}$ operator.
The resulting score $S_i'=w^\top x_i$ therefore quantifies how well each domain $x_i$ captures the robust consensus direction.
\end{remark}

To enable \emph{multi-modal} integration, we reformulate the primal objective \eqref{eq:primal} into a dual-form lower bound, allowing multiple modalities to be aligned in a shared latent space.
Through introducing latent variables $\hs_i$ and the Fenchel-Young inequality
$\frac{1}{2\lambda}e^2+\frac{\lambda}{2} {\hs}^2 \geq e \hs, \, \forall e, \hs \in \R^k$ ~\citep{rockafellar1974conjugate,suykens2017deep}, we can express the primal problem of single modality as:
\begin{equation} \label{eq:rkm}
\begin{aligned}
J &= \frac{1}{2\lambda} \sum_{i=1}^k e_i^2 + \frac{1}{2} \norm{w}_\fro^2 \quad \text{s.t. } e_i = 1 - w^\top x_i, \, i\in [k]\\
&\geq \sum_{i=1}^k e_i \hs_i - \frac{\lambda}{2} \norm{\hs}_\fro^2 + \frac{1}{2} \norm{w}_\fro^2 \\
&= \sum_{i=1}^k (1 - w^\top x_i) \hs_i - \frac{\lambda}{2} \norm{\hs}_\fro^2 + \frac{1}{2} \norm{w}_\fro^2 =: J_\text{SM},
\end{aligned}
\end{equation}
where $\hs = [\hs_1, \ldots, \hs_k] \in \mathbb{R}^k$ is the vector of latent variables.
By analyzing the stationary conditions of the lower-bound single-modality objective function $J_\text{SM}$
and eliminating the primal variable $w$, we obtain the following solution in the latent variables:
$\hs = (\Ks + \lambda I_k)^{-1} 1_k,$
where $\Ks \in \R^{k \times k}$ is the domain affinity kernel matrix defined by ${\Ks}_{ij}=x_i^\top x_j$, and $I_k \in \R^{k\times k}$ and $\mathbf{1}_k\in\R^{k\times 1}$ denote the identity matrix and all-ones vector respectively.
Then, the unimodal domain score $S_i'$ can be expressed as:
$
S_i' = [\Ks (\Ks + \lambda I)^{-1} 1_k]_i,
$
which is consistent with the 
covariance-based solution derived 
from the original problem~\eqref{eq:primal}, with derivation details in \cref{supp:subsec:single_modality,supp:subsec:latent_space}.

Importantly, such dual structure with explicit latent variables $\hs_i$ in \eqref{eq:rkm} facilitates the extension to \textbf{multi-modal integration}.
Let $w^{[v]}$ denote the projection weight for modality $v\in[V]$. 
Define the alignment objective for each modality $v$ as $J_\text{SM}^{[v]}(w^{[v]},\h)$.
We express the multi-modal scoring objective as
\begin{equation}
\label{eq:mm_objective0}
\begin{split}
    \tilde{J}_{\text{MM}} &= \sum_{v=1}^V J_\text{SM}^{[v]} (w^{[v]},\h)
        = \sum_{v=1}^V \sum_{i=1}^k (1 - (w^{[v]})^\top x_i^{[v]}) \h_i - \frac{\lambda}{2} \sum_{v=1}^V \|\h\|^2_\fro + \frac{1}{2} \sum_{v=1}^V \norm{w^{[v]}}_\fro^2,
\end{split}
\end{equation}
which implicitly sets ${\hs}^{[1]} = \cdots = {\hs}^{[V]} = \h$, connecting domain embeddings across modalities in the shared latent space, thereby realizing \emph{inter-modality couplings}.

\paragraph{Interpretation.}
The dual multi-modal objective~\eqref{eq:mm_objective0} jointly optimizes the scores $({w^{[v]}})^\top x_i^{[v]}$ for all domains and modalities. 
The first term of~\eqref{eq:mm_objective0} can be interpreted 
as an energy function~\citep{bengio2009learning} penalizing high-energy solutions, i.e., large 
$(1 - {(w^{[v]})}^\top x_i^{[v]})$ disagreements. 
The dual variable $\alpha_i$ serves as a consensus variable: large values push all 
modality weights to reduce disagreement for that domain. 
The remaining terms serve as regularization controlling the weight norm and the distribution of the dual variables.

\subsection{Multi-modal scores with missing modalities} 
\label{subsec:miss_modality}

Accommodating data with incomplete modality coverage is a key challenge in VLM training.
For instance, with vision and text modalities, some domains may only contain text, while others may present both.
This scenario commonly occurs in practical settings as VLMs are typically trained on a mix of multi-modal and pure text data to retain the model's dialogue capabilities.

To address the issue of missing modalities, we adjust the projection errors appropriately. 
To be specific, we set $x^{[v]}_i=0_d$ along with zero as the target for the missing modality $v$ in the $i$-th domain, 
which ensures that domains lacking a modality do not contribute to the scoring objective.
Incorporating the modality indicator $\delta_i^{[v]}$, the final multi-modal scoring objective from \eqref{eq:mm_objective0} can be expressed as:
\begin{equation} \label{eq:mm_objective}
\begin{split}
    J_{\text{MM}} &= \sum_{v=1}^V \sum_{i=1}^k \Big[ (\delta_i^{[v]} - (w^{[v]})^\top x_i^{[v]}) \h_i - \frac{\lambda}{2} \h_i^2 \Big] + \frac{1}{2} \sum_{v=1}^V \norm{w^{[v]}}_\fro^2,
\end{split}
\end{equation}
where $\delta_i^{[v]} \in \{0,1\}$ indicates the existence of modality $v$ in domain $D_i$.

\begin{algorithm}[t]
\caption{Multi-modal Data Mixtures (\ours)}
\label{alg:mm-algo}
\begin{algorithmic}[1]
    \State \textbf{Input:} Number of domains $k$, number of modalities $V$, 
    domain embeddings $x_i^{[v]}\in\R^d$ for available modalities,
    and regularization parameter $\lambda$.

    \State Preprocessing: 1) Set $x_i^{[v]}=0_d$ for missing modality $v$ in domain $i$;
    2) Construct target vector $\delta \in \R^k$ where $\delta_i$ counts available modalities in domain $i$.
    \State Modality kernels: $K^{[v]} = [(x_i^{[v]})^\top x_j^{[v]}]_{i,j=1}^k$.
    \State Multi-modal domain affinity: $\K = \sum_{v=1}^V \Ks^{[v]}$.
    \State Alignment scores: 
    $S_i^{[v]} = \left[K^{[v]} (\K + \lambda I)^{-1} \delta\right]_i$.
    \State Domain weights: $\myalpha_i = \frac{\exp(\sum_{v=1}^V S_i^{[v]})}{\sum_{j=1}^k \exp(\sum_{v=1}^V S_j^{[v]})}$.
    \State \textbf{Output:} Domain weights $\myalpha = [\myalpha_1, \dots, \myalpha_k]$.
\end{algorithmic}
\end{algorithm}

We obtain the solution in the shared latent variables $\h$ in the multi-modal setting by analyzing the stationary conditions of \eqref{eq:mm_objective} through the derivation in \cref{supp:subsec:multi-modality}, summarized in the following Proposition.

\begin{proposition}[Multi-modal scores] \label{prop:multi-modal}
Define the multi-modal kernel matrix as $\K \in \mathbb{R}^{k \times k}$ with entries $\Ks_{\mathrm{MM}_{\scriptstyle ij}} = \sum_{v=1}^V \Ks_{ij}^{[v]}$, where $\Ks_{ij}^{[v]} =(x_i^{[v]})^\top x_j^{[v]}$.
The optimal latent variables for the multi-modal objective are given by:\looseness=-1
\begin{equation} \label{eq:h:multi-modal}
\h = (\K + \lambda {I})^{-1} \delta,
\end{equation}
where  $\delta = [\delta_1, ..., \delta_k]^\top$ with entries $\delta_{i} = \sum_{v=1}^V \delta_i^{[v]}$.
Note that $\delta_i$ is always a positive constant since all domains have at least one modality.
At optimality, the domain alignment score $S_i^{[v]}={w^{[v]}}^\top x_i^{[v]}$ for modality $v$ of domain $D_i$ in kernel representation is:
\begin{equation} \label{eq:final_score}
S_i^{[v]} = \left[\Ks^{[v]} (\K+\lambda I)^{-1} \delta\right]_i,
\end{equation}
with $\K$ realizing the modality couplings.
\end{proposition}

A high score $S_i^{[v]}$ indicates that the $v$-th modality of domain $D_i$ has a large component along the cross-modal consensus direction expressed through multi-modal coupling coefficients $\h$.
To obtain the final domain distribution, we aggregate the scores $S_i^{[v]}$ across all modalities for each domain $D_i$. 
The resampling distribution $\myalpha$ for VLM training is then obtained by softmax-normalizing the scores:
$\myalpha_i = \frac{\exp(\sum_{v=1}^V S_i^{[v]})}{\sum_{j=1}^k \exp(\sum_{v=1}^V S_j^{[v]})}$.

\hfill
\begin{remark}[Spectral characterization]
The multi-modal kernel matrix \(K_{\mathrm{MM}}\) induces a spectral decomposition of the domain relationships.
Let \(K_{\mathrm{MM}} = U \Sigma U^\top\) with eigenvalues \(\{\sigma_j\}_{j=1}^k\) and eigenvectors \(\{u_j\}_{j=1}^k\). In \cref{supp:subsec:spectral_soft}, we show that the resulting score for domain \(D_i\) admits the expansion $S_i = \sum_{j=1}^k \frac{\sigma_j}{\sigma_j + \lambda} \, (u_j^\top \delta)\,(u_j)_i$.
The \ours score applies a spectral soft-thresholding to the multi-modal kernel: directions associated with large eigenvalues (signal-dominated components) are preserved, and small-eigenvalue directions (noise-dominated components) are attenuated.
\end{remark}

\paragraph{Computational complexity \& practical implementation.}
Our algorithm is summarized in \cref{alg:mm-algo}. Computing embeddings $x_i^{[v]}$ requires a cheap inference pass through the model from the previous stage.
The kernel score computation \eqref{eq:final_score} is computationally cheap given the typically small number of domains $k$ used in VLM training
(detailed complexity analysis is given in \cref{remark:cost_of_score}).
Notably, \ours operates independently of the VLM optimization procedure, enabling noninvasive integration into diverse pipelines by simply adjusting sampling weights.
This approach is a key advantage in the VLM setting where many differing training pipelines are commonly used.

%% file: fig/pipeline.tex
\begin{figure}[t]
    \centering
    \includegraphics[width=\textwidth]{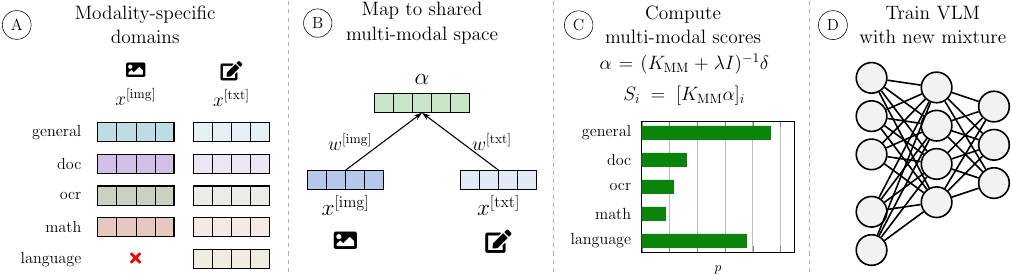}
    \caption{\textbf{Pipeline of multi-modal data mixing for VLM training.}
    \textcircled{\scriptsize A} Modality-specific embeddings $x_i^{[v]}$ are extracted from the midstage trained model for each domain.
    Some domains may lack certain modalities (e.g., the language domain has no image data). 
    \textcircled{\scriptsize B} The $k$ domains are then mapped to a shared multi-modal space by the coupling latent variables $\h$ of the multi-modal alignment objective \eqref{eq:mm_objective}.
    \textcircled{\scriptsize C} The multi-modal kernel matrix $\K$ is computed as the pairwise inner products between domain embeddings across modalities via \eqref{eq:h:multi-modal}.
    Finally, \eqref{eq:final_score} is applied to $\K$ and $\h$ to obtain score $S_i, \, i=1,\ldots,k$ indicating the multi-modal alignment of each domain.
    A resampling non-uniform distribution $p$ is obtained by softmax-normalizing the scores.
    \textcircled{\scriptsize D} Finally, image-text instruction tuning of the target VLM is carried out by sampling according to the obtained data mixture $p$.
    }
    \label{fig:pipeline}
\end{figure}

%% file: sections/8-Experiments.tex
We conduct a comprehensive empirical evaluation of our multi-modal data mixing method for visual instruction tuning of LLaVA-OneVision~\citep{li2024llavaonevisioneasyvisualtask} across diverse VLM benchmarks.
We follow the standard domain construction of~\citep{li2024llavaonevisioneasyvisualtask}, with each domain corresponding to a target VLM skill
providing an ideal testbed for data mixing strategies \citep{laurenccon2024matters,dong2025scalablevisionlanguagemodel}. Furthermore, the data incorporate text, image, and video modalities and realistically reflect practical challenges such as the inconsistent availability of modalities across domains.

First, we evaluate our method on the stage-2 image-text instruction tuning~\citep{li2024llavaonevisioneasyvisualtask}, which contains five domains including text and image modalities, and compare average accuracy on multiple benchmarks to other mixing baselines.
\ours accelerates training at expert-tuned performance with marginal computational cost.
Further, we explore the transferability of our domain weights across model sizes and architectures on Qwen-VL.
Then, we introduce an additional video modality,
showing that our algorithmic mixing naturally extends to more complex multi-modal settings, yielding consistent improvements and providing an efficient, scalable alternative to costly expert tuning.

\paragraph{Training setup.}  
We train LLaVA-OneVision 0.5B and 7B models with batch size 128, sequence length 8192, and learning rate $10^{-5}$ with cosine decay. 
In \cref{subsec:exp:text_image}, models are trained for 4500 steps following \cite{li2024llavaonevisioneasyvisualtask} s.t. each example is used only once.
The training data consists of five domains: General, Doc/Chart/Screen, Math/Reasoning, General OCR, and Language. 
The first four domains are structured as image-text pairs, while the Language domain includes text data only, lacking the image modality.
In \cref{subsec:exp:add_video}, we introduce an additional VideoQA domain with video-text data and train for 3000 steps to further test our method's multi-modal capabilities.

\paragraph{Baselines.}
\textsc{Uniform} is the cost-free mixture assigning equal weights $p_i = \frac{1}{k}$, which, despite its simplicity, can be a strong baseline~\citep{michel2021balancing,fan_doge_2024}.
\textsc{Human} corresponds to the domain weights manually optimized by the authors of \citep{li2024llavaonevisioneasyvisualtask}.
\textsc{Text}, \textsc{Image}, and \textsc{Video} represent weights derived solving \cref{eq:primal} based on embeddings from a single modality. If a domain lacks a specific modality, its corresponding weight is set to zero.
\textsc{Avg} averages the domain weights of all single modalities. For example, $\textsc{Avg} = \frac{1}{2} (\textsc{Text}+\textsc{Image})$ in \cref{subsec:exp:text_image} and $\textsc{Avg} = \frac{1}{3} (\textsc{Text}+\textsc{Image} + \textsc{Video})$ in \cref{subsec:exp:add_video}.
Moreover, \textsc{Fused} are the domain weights computed from the fused multi-modal embedding, which is generated by the VLM after processing all modalities as a unified sequence.
\ours computes the domain weights through \cref{eq:final_score}.
The processes of embedding extraction and domain weight assignment are detailed in \cref{supp:subsec:embedding_extract}.

\paragraph{Evaluation benchmarks.} 
We use various benchmarks for VLM evaluation in diverse tasks
and they can be categorized into three classes following~\citep{li2024llavaonevisioneasyvisualtask}:
\textit{(1) Chart, Diagram, and Document Understanding}.
Charts and diagrams are key formats for visual information expression. We evaluate on AI2D~\citep{kembhavi2016ai2d}, ChartQA~\citep{masry2022chartqa}, DocVQA~\citep{mathew2021docvqa}, and InfoVQA~\citep{mathew2022infographicvqa}, and OCRBench~\citep{liu2024ocrbench} for text recognition.
\textit{(2) Perception and Multi-discipline Reasoning}.
For more complex visual detection scenarios, we also evaluate on more challenging multi-disciplinary visual-language reasoning tasks. Specifically, we follow the multi-modal benchmarks of MME~\citep{yin2023survey}, MMBench~\citep{liu2023mmbench}, and reasoning benchmarks including MathVerse~\citep{zhang2024mathverse}, MMMU~\citep{yue2023mmmu}, and ScienceQA~\citep{lu2022learn}.
\textit{(3) Real-world Understanding and Multi-modal Chatbots}. 
We also benchmark the capability of VLMs as a general-purpose assistant in the real world with specific benchmarks, including RealworldQA~\citep{grok15v} and MMStar~\citep{chen2024we}.
In \cref{tab:eval_video}, we add two video benchmarks: Video-MMMU~\citep{hu2025video} and MVBench~\citep{li2024mvbench}.
We use the LLMs-Eval library \citep{zhang2024lmms} for evaluation.

\subsection{\ours improves both 0.5B and 7B VLMs}
\label{subsec:exp:text_image}
We train LLaVA-OneVision-0.5B using the domain reweighting strategies discussed in Baselines above during the single-image (i.e., no video) training phase. 
The domain weights are visualized in \cref{fig:weights_acc} (top) and listed in \cref{supp:subsec:domain_weights}.
These models are evaluated on ten diverse benchmarks, presenting the 0-shot accuracy in \cref{tab:eval_si_0.5B}.
Our \ours achieves the highest average score across all benchmarks, bringing a 1.24\% improvement over \textsc{Uniform}.
{Importantly, it even surpasses \textsc{Human} that requires expensive and non-scalable grid searches on given domains, while our method can find mixtures with \emph{negligible computational cost}.}
Remarkably, \ours learns faster: as shown in \cref{fig:weights_acc} (bottom), it outperforms \textsc{Uniform} with just 56\% steps and matches \textsc{Human} with 78\% steps, corresponding to $1.8\times$ and $1.28\times$ speedup factors, respectively.

For further analysis in \cref{tab:eval_si_0.5B}, \ours outperforms 1) \textsc{Avg} that handles modalities independently, and 2) \textsc{Fused} that uses the fused embeddings from VLM with all available modalities as input. 
Moreover, \ours also surpasses unimodal strategies that ignore the information from other modalities, as demonstrated in \cref{supp:subsec:acc_single_modality}.
This indicates the importance of distinctly considering the contributions of each modality and addressing the missing modal data specifically.
Moreover, our ablation studies in \cref{supp:ablations} demonstrate the robustness of \ours's domain weights. 
We further ablate with an Orthogonal Score variant explicitly prioritizing feature dissimilarity, as shown in \cref{supp:subsec:orthogonal_score}. The observed results suggests that emphasizing domain heterogeneity alone is less effective in this setting than the signal captured by our score.

In addition, the downweighted domains do not result in sacrificing the model's corresponding capabilities. 
Specifically, although \ours downweights Math and OCR compared to \textsc{Uniform}, it preserves the capabilities on MathVerse and OCRBench in \cref{tab:eval_si_0.5B}.
This suggests that our method supports positive transfer across domains: \textit{emphasizing a subset of high-score domains can promote emergent capabilities in others}, even if they receive less training weight.

\input{fig/weights_acc}


\begin{table*}[t]
    \centering
    \caption{\textbf{Comparison of data mixing strategies for LLaVA-0.5B image-text instruction tuning.}
    Results are reported as 0-shot accuracy across ten evaluation benchmarks.
    We compare our \ours against baselines: \textsc{Uniform} (equal weights), \textsc{Human} (manual weights), \textsc{Avg} (averaged single-modality weights), and \textsc{Fused} (weights from input concatenation).
    }
    \begin{tabularx}{\textwidth}{lCCCCC}
    \toprule 
    Benchmark & \textsc{Uniform} & \textsc{Human} & \textsc{Avg} & \textsc{fused} & \ours \\
    \midrule
        AI2D        & $42.78_{\pm 0.04}$ & $43.75_{\pm 0.01}$ & $45.50_{\pm 0.02}$ & $44.59_{\pm 0.05}$ & $43.52_{\pm 0.09}$ \\[1mm]
        DocVQA      & $42.90_{\pm 0.02}$ & $42.66_{\pm 0.00}$ & $42.44_{\pm 0.03}$ & $42.67_{\pm 0.01}$ & $42.92_{\pm 0.02}$ \\[1mm]
        InfoVQA     & $22.25_{\pm 0.03}$ & $22.61_{\pm 0.07}$ & $22.43_{\pm 0.04}$ & $23.50_{\pm 0.03}$ & $22.13_{\pm 0.05}$ \\[1mm]
        MathVerse   & $18.27_{\pm 0.03}$ & $17.26_{\pm 0.11}$ & $18.32_{\pm 0.06}$ & $19.29_{\pm 0.08}$ & $18.91_{\pm 0.07}$ \\[1mm]
        MMBench     & $36.34_{\pm 0.00}$ & $40.21_{\pm 0.04}$ & $39.86_{\pm 0.08}$ & $37.71_{\pm 0.12}$ & $42.44_{\pm 0.04}$ \\[1mm]
        MMStar      & $33.45_{\pm 0.06}$ & $36.04_{\pm 0.10}$ & $33.50_{\pm 0.14}$ & $34.44_{\pm 0.20}$ & $35.88_{\pm 0.03}$ \\[1mm]
        MMMU        & $30.00_{\pm 0.16}$ & $29.67_{\pm 0.31}$ & $29.00_{\pm 0.09}$ & $29.22_{\pm 0.21}$ & $29.78_{\pm 0.16}$ \\[1mm]
        ScienceQA   & $62.42_{\pm 0.02}$ & $65.84_{\pm 0.02}$ & $64.80_{\pm 0.04}$ & $63.46_{\pm 0.09}$ & $64.50_{\pm 0.01}$ \\[1mm]
        OCRBench    & $45.30_{\pm0.05}$ & $44.60_{\pm 0.09}$ & $45.30_{\pm 0.06}$ & $43.50_{\pm 0.09}$ & $45.80_{\pm 0.05}$ \\[1mm]
        RealworldQA & $46.27_{\pm 0.18}$ & $44.05_{\pm 0.06}$ & $45.49_{\pm 0.10}$ & $45.36_{\pm 0.12}$ & $46.54_{\pm 0.06}$ \\[1mm]
    \midrule
       \rowcolor{gray!15} Average     & $38.00_{\pm 0.09}$ & $38.67_{\pm 0.12}$ & $38.66_{\pm 0.08}$ & $38.37_{\pm 0.12}$ & $\mathbf{39.24}_{\pm 0.08}$ \\[1mm]
       \rowcolor{gray!15} Number over \textsc{Uniform} & - & $5/10$ & $6/10$ & $6/10$ & $8/10$  \\
    \bottomrule
    \end{tabularx}   
    \label{tab:eval_si_0.5B}
\end{table*}

\paragraph{Domain weights transfer to larger VLMs.}
Recent research on data mixing in text-only LLMs shows that domain weights derived from smaller models can be effectively transferred to larger ones \citep{xie_doremi_2023,fan_doge_2024,liu2024regmix}. 
We investigate this phenomenon in VLMs by training 7B models with the domain weights obtained from 0.5B models.
The evaluation results are presented in \cref{tab:eval_si_7B}.
Remarkably, \ours maintains its performance advantage over baselines even at this increased model scale, outperforming \textsc{Uniform} on 8 out of the 10 benchmarks.

\paragraph{Marginal computational cost.}
\begin{wraptable}{r}{0.38\linewidth}
\vspace{-14pt}
\centering
\caption{Computational cost is negligible relative to full model training. Cost in H100 GPU hours.}
\label{tab:cost}
\vspace{-4.5pt}
\begin{tabular}{l c}
\toprule
\textbf{Component} & \textbf{Cost (h)} \\
\midrule
Embedding extraction & 0.58 \\
Score computation & 0.01 \\
Total & 0.59 \\
\midrule
Training (0.5B) & 90 \\
Training (7B) & 620 \\
\bottomrule
\end{tabular}
\vspace{-5pt}
\end{wraptable}
The computational complexity of our method is discussed in~\cref{sec:method}.
In practice,
\textit{(i)} embedding extraction is a fast inference-only process. In our experiments, this step takes 35 minutes on a single H100 GPU.
\textit{(ii)} Alignment score computation via \eqref{eq:final_score} completes in seconds since the number of domains is small.
The cost of our weight computation is marginal compared to the 90 and 620 GPU hours required to train 0.5B and 7B VLMs, respectively.
Our approach substantially reduces the need for expensive, time-consuming manual tuning of data mixtures, which is a key bottleneck in current VLM development.

\begin{table*}[t]
    \centering
    \caption{\textbf{Transfer weights from LLaVA-0.5B to LLaVA-7B for image-text instruction tuning.}
    Results are reported as 0-shot accuracy across ten benchmarks.
    \ours improves performance on 8 out of 10 benchmarks over \textsc{Uniform}.
    }
    \begin{tabularx}{\textwidth}{lCCCCC}
    \toprule
    Benchmark & \textsc{Uniform} & \textsc{Human} & \textsc{Avg} & \textsc{Fused} & \ours \\
    \midrule
        AI2D        & $74.48_{\pm 0.04}$ & $74.03_{\pm 0.11}$ & $75.10_{\pm 0.08}$ & $75.74_{\pm 0.05}$ & $75.58_{\pm 0.09}$ \\[1mm]
        DocVQA      & $57.91_{\pm 0.08}$ & $58.64_{\pm 0.05}$ & $58.28_{\pm 0.12}$ & $57.29_{\pm 0.15}$ & $58.32_{\pm 0.03}$ \\[1mm]
        InfoVQA     & $34.76_{\pm 0.15}$ & $35.91_{\pm 0.09}$ & $36.95_{\pm 0.07}$ & $36.06_{\pm 0.11}$ & $36.23_{\pm 0.18}$ \\[1mm]
        MathVerse   & $29.31_{\pm 0.09}$ & $26.85_{\pm 0.14}$ & $27.33_{\pm 0.18}$ & $28.68_{\pm 0.06}$ & $28.55_{\pm 0.12}$ \\[1mm]
        MMBench     & $75.69_{\pm 0.02}$ & $76.12_{\pm 0.03}$ & $76.23_{\pm 0.05}$ & $75.77_{\pm 0.08}$ & $75.74_{\pm 0.06}$ \\[1mm]
        MMStar      & $49.04_{\pm 0.11}$ & $50.26_{\pm 0.16}$ & $50.44_{\pm 0.09}$ & $49.46_{\pm 0.14}$ & $50.19_{\pm 0.10}$ \\[1mm]
        MMMU        & $46.33_{\pm 0.21}$ & $46.78_{\pm 0.18}$ & $46.78_{\pm 0.22}$ & $46.78_{\pm 0.17}$ & $46.89_{\pm 0.15}$ \\[1mm]
        ScienceQA   & $87.31_{\pm 0.06}$ & $90.38_{\pm 0.02}$ & $89.53_{\pm 0.04}$ & $85.52_{\pm 0.09}$ & $90.23_{\pm 0.07}$ \\[1mm]
        OCRBench    & $56.80_{\pm 0.13}$ & $57.30_{\pm 0.08}$ & $56.70_{\pm 0.11}$ & $56.60_{\pm 0.08}$ & $57.90_{\pm 0.14}$ \\[1mm]
        RealworldQA & $58.17_{\pm 0.10}$ & $57.91_{\pm 0.12}$ & $56.99_{\pm 0.14}$ & $57.65_{\pm 0.10}$ & $57.47_{\pm 0.05}$ \\[1mm]
    \midrule
        \rowcolor{gray!15} Average     & $56.98_{\pm 0.11}$ & $57.42_{\pm 0.11}$ & $57.43_{\pm 0.13}$ & $56.96_{\pm 0.11}$ & $\mathbf{57.71}_{\pm 0.11}$ \\[1mm]
        \rowcolor{gray!15} Number over \textsc{Uniform} & - & $7/10$ & $7/10$ & $5/10$ & $8/10$  \\
    \bottomrule
    \end{tabularx}    
    \label{tab:eval_si_7B}
\end{table*}

\subsection{Scale to more complex tri-modal settings}
\label{subsec:exp:add_video}
We further demonstrate the flexibility of \ours in more complex multimodal scenarios by adding a VideoQA domain that introduces video–text data.
This creates a total of six domains with three modalities: text, image, and video.
Our domain weights for this new configuration are shown in \cref{fig:weights_acc_video} (top) and fully reported in \cref{supp:subsec:domain_weights_video}.
We train 0.5B and 7B models with new domain weights and evaluate models on both image-only benchmarks (same as \cref{subsec:exp:text_image}) and two video benchmarks, MVBench \citep{li2024mvbench} and Video-MMMU \citep{hu2025video}.

{The results in \cref{tab:eval_video} demonstrate that \ours maintains its superiority over \textsc{Uniform} in the more complex tri-modal setting.}
Notably, \ours matches \textsc{Uniform} performance in only 33\% steps, as shown in \cref{fig:weights_acc_video} (bottom), resulting in a $3\times$ average speedup.
In addition, \ours outperforms both \textsc{Avg} and \textsc{Fused} on average.
This validates the effectiveness of our multi-modal construction compared to single-modality mixtures and simple early fusion.
Crucially, this experiment highlights the \emph{extensibility} 
of our method to richer multi-modal configurations where manual expert-tuning becomes increasingly impractical.

\input{fig/weights_acc_video}

\begin{table}[!ht]
    \centering
    \caption{\textbf{Comparison of data mixtures for LLaVA-0.5B/7B video-image-text instruction tuning.}
    Results are reported as 0-shot accuracy across twelve evaluation benchmarks.
    \ours achieves the best average performance on two model sizes.
    Results with standard deviations are in \cref{supp:subsec:acc_video_std}.
    }
    \begin{tabularx}{\textwidth}{
    l 
    >{\hsize=0.9\hsize}C 
    >{\hsize=0.9\hsize}C 
    >{\hsize=0.9\hsize}C 
    >{\hsize=1.3\hsize}C
    >{\hsize=0.9\hsize}C 
    >{\hsize=0.9\hsize}C 
    >{\hsize=0.9\hsize}C 
    >{\hsize=1.3\hsize}C
    }
    \toprule
    \multirow{2}{*}{Benchmark} & \multicolumn{4}{c}{0.5B} & \multicolumn{4}{c}{7B} \\
    \cmidrule(lr){2-5} \cmidrule(lr){6-9}
    & \textsc{Unif.} & \textsc{Avg} & \textsc{Fused} & \ours & \textsc{Unif.} & \textsc{Avg} & \textsc{Fused} & \ours \\
    \midrule
        AI2D        & $41.68$ & $42.81$ & $42.84$ & $42.88$ & $71.83$ & $72.41$ & $72.83$ & $72.15$ \\[1mm]
        DocVQA      & $42.20$ & $41.68$ & $41.29$ & $42.54$ & $56.47$ & $56.42$ & $55.67$ & $57.51$ \\[1mm]
        InfoVQA     & $21.65$ & $21.97$ & $21.17$ & $22.40$ & $35.74$ & $34.65$ & $34.40$ & $35.89$ \\[1mm]
        MathVerse   & $15.61$ & $15.62$ & $17.77$ & $15.10$ & $25.63$ & $25.52$ & $24.75$ & $26.40$ \\[1mm]
        MMBench     & $34.36$ & $26.80$ & $35.14$ & $34.45$ & $71.05$ & $75.52$ & $73.28$ & $74.57$ \\[1mm]
        MMStar      & $30.43$ & $35.54$ & $36.14$ & $33.97$ & $48.18$ & $49.03$ & $46.55$ & $48.79$ \\[1mm]
        MMMU        & $30.00$ & $29.78$ & $30.44$ & $29.78$ & $45.67$ & $45.11$ & $44.78$ & $45.56$ \\[1mm]
        ScienceQA   & $60.29$ & $60.29$ & $59.40$ & $61.03$ & $83.44$ & $86.07$ & $83.29$ & $87.26$ \\[1mm]
        OCRBench    & $45.30$ & $43.20$ & $46.60$ & $45.00$ & $56.50$ & $56.90$ & $57.60$ & $57.20$ \\[1mm]
        RealworldQA & $47.19$ & $46.41$ & $46.27$ & $47.32$ & $57.91$ & $56.99$ & $59.22$ & $57.39$ \\[1mm]
        Video-MMMU  & $13.78$ & $13.78$ & $12.78$ & $13.84$ & $29.78$ & $30.56$ & $29.11$ & $30.33$ \\[1mm]
        MVBench     & $36.67$ & $36.50$ & $37.02$ & $40.70$ & $52.73$ & $51.58$ & $53.12$ & $53.60$ \\[1mm]
    \midrule
        \rowcolor{gray!15} Average     & $34.93$ & $34.53$ & $34.74$ & $\mathbf{35.75}$ & $52.91$ & $53.39$ & $52.88$ & $\mathbf{54.40}$ \\[1mm]
        \rowcolor{gray!15} \# over \textsc{Unif.} & - & 4/12 & 7/12 & 9/12 & - & 6/12 & 5/12 & 10/12 \\
    \bottomrule
    \end{tabularx}
    \label{tab:eval_video}
\end{table}

\paragraph{Domain weights transfer to Qwen-VL-2B.}
To explore the cross-architecture transferability of \ours, we apply domain weights obtained from LLaVA-0.5B to Qwen-VL-2B \citep{Qwen2-VL} in this tri-modal setup.
The results in \cref{supp:subset:qwen-vl} show that \ours maintains its benefit compared with other baselines also on Qwen-VL-2B.
This transferability suggests that \ours captures intrinsic data characteristics that generalize across architectures, indicating its potential utility for broader model development.

%% file: fig/weights_acc.tex
\begin{figure}[t]
    \centering
    \includegraphics[width=\textwidth]{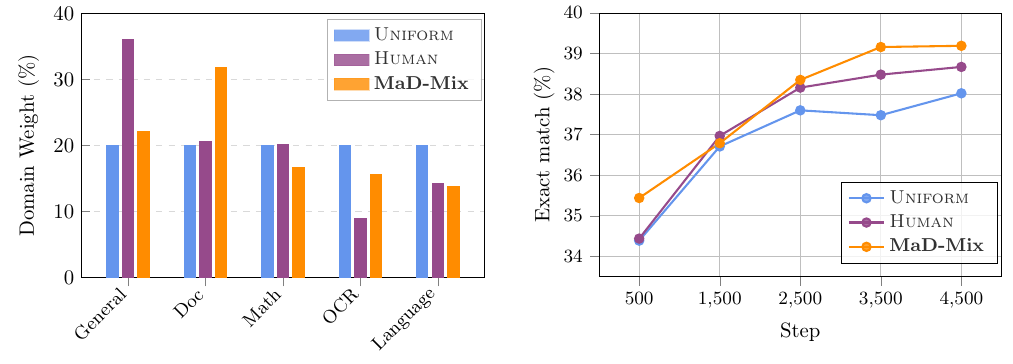}
    \caption{
        \textbf{Comparison of different data mixture strategies in the image-text instruction tuning.} (Left) Domain weights for \textsc{uniform}, \textsc{human}, and \ours. (Right) Zero-shot average downstream accuracy of 0.5B models, where \ours achieves consistent improvement.
    }
    \label{fig:weights_acc}
\end{figure}

%% file: fig/weights_acc_video.tex
\begin{figure}[t]
    \centering
    \includegraphics[width=\textwidth]{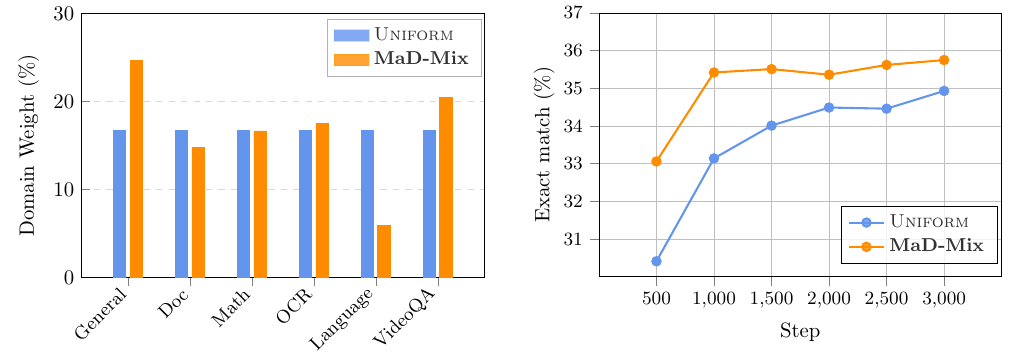}
    \caption{
        \textbf{Comparison of different data mixtures in the video-image-text instruction tuning.}
        (Left) Domain weights for \textsc{uniform} and \ours. (Right) Zero-shot average downstream accuracy of 0.5B models, where \ours outperforms \textsc{Uniform} during the whole training process.
    }
    \label{fig:weights_acc_video}
\end{figure}

%% file: sections/9-Conclusion.tex
This paper presents a principled approach to the key problem of algorithmically determining sampling weights across pre-defined domains for vision-language model training with negligible additional computational cost. 
Our formulation through modality-aware fusion with coupling inter-modal variables addresses fundamental challenges in VLM training: handling missing modalities, learning cross-modal features, and determining domain weights without costly grid searches. 
Empirical evaluations demonstrate that our method improves training efficiency compared to both uniform and manually-tuned mixtures across diverse VLM benchmarks on average. 
Our approach allows direct integration with existing diverse VLM training pipelines and makes it valuable for practical applications, offering a path towards more data- and compute-efficient VLM training.

%% file: sections/10-Appendix.tex
\section{Problem formulation} \label{supp:problem_formulation}

\subsection{Formulation for single modality setting}
\label{supp:subsec:single_modality}
Let the data mixture problem consist of $k$ data domains and their domain embeddings $x_i \in \R^d$ with their target $y_i\in\R$, $i=1,\ldots,k$.
We first write the primal domain scoring problem for single modality:
\begin{equation} 
\label{eq:primal:app}
\min_{w,e} \frac{1}{2\lambda} \sum_{i=1}^k e_i^2 + \frac{1}{2} \|w\|^2 \quad \text{s.t. } e_i = y_i - w^\top x_i, \, i=1,\ldots,k,
\end{equation}
where $w \in \R^d$, $e = [e_1, \ldots, e_k]\in \R^k$ are the projections, $\lambda > 0$ is a regularization constant.

From the Lagrangian with dual variables $\nu$:
$$
    \cL(w,e;\nu) = \frac{1}{2\lambda} \sum_{i=1}^k e_i^2 + \frac{1}{2} \|w\|^2 - \sum_{i=1}^k \nu_i (e_i - y_i + w^\top x_i),
$$
one takes the conditions for optimality, which are given as
$$
\left\{
\begin{aligned}
\frac{\partial \mathcal{L}}{\partial w} &= w - \sum_{i=1}^k \nu_i x_i = 0 
  \quad \implies \quad w = \sum_{i=1}^k \nu_i x_i, \\[1ex]
\frac{\partial \mathcal{L}}{\partial e_i} &= \frac{1}{\lambda} e_i - \nu_i = 0 
  \quad \implies \quad e_i = \lambda \nu_i, \quad \forall i \\[1ex]
\frac{\partial \mathcal{L}}{\partial \nu_i} &= e_i - y_i + w^\top x_i = 0 
  \quad \implies \quad \lambda \nu_i - y_i + \sum_{i=j}^k \nu_j x_j^\top x_i = 0 \quad \forall i.
\end{aligned}
\right.
$$
Eliminating $w$ in the last condition gives the dual solution:
$$
\begin{aligned}
K \nu = y - \lambda \nu, \\
(K+\lambda I)\nu = y, \\
\nu = (K+\lambda I)^{-1} y,
\end{aligned}
$$
where we defined the kernel matrix as $K = [x_i^\top x_j]_{i,j=1}^k$, and the target vector $y=[y_1,\dots, y_k]^\top$.

We are now ready to define the score of domain $i$ as $S_i' = w^\top x_i$ in its kernel form:
\begin{equation} \label{eq:krls:app}
S_i' = w^\top x_i 
    = \left(\sum_{j=1}^k \nu_j x_j\right)^\top x_i 
    = \left[ K(K+\lambda I)^{-1} y \right]_i.
\end{equation}

\subsubsection{Primal and dual score representations}
We can write \cref{eq:primal:app} in the unconstrained form:
\[
\min_{w} \; \frac{1}{2\lambda} \sum_{i=1}^k (y_i - w^\top x_i)^2 + \frac{1}{2} \|w\|^2.
\]
This is a ridge regression problem where the target vector is $y=[y_1,\dots,y_k]^\top$.
Let \(X\) be the \(k \times d\) data matrix with rows \(x_1^\top, x_2^\top, \dots, x_k^\top\). Then the objective becomes:
\[
\min_{w} \; \frac{1}{2\lambda} \| y - Xw \|^2 + \frac{1}{2} \|w\|^2.
\]
The solution to this ridge regression problem is:
\[
w = (X^\top X + \lambda I)^{-1} X^\top y.
\]
The score for domain $i$ is $S_i' = w^\top x_i$. The vector of scores can be computed as $S' = Xw$. Substituting the expression for $w$:
$$
S_i' = [X (X^\top X + \lambda I)^{-1} X^\top y]_i.
$$
This is equivalent to \eqref{eq:krls:app} by standard matrix identity, i.e., Woodbury identity.
The following remark summarizes the computational aspect of the primal and dual representations of the score.

\begin{remark}[Efficient computation of the score]
\label{remark:cost_of_score}
    The primal solution is written in terms of the covariance $X^\top X$, while the dual solution is in terms of the kernel matrix $XX^\top$.
    In the context of data mixture with large VLMs, the embedding dimension $d$ may be very large, so it is computationally advantageous to work in the dual with complexity $\cO(k^3)$ where the number of data domains $k$ is typically much smaller.
    We report the computational time of the defined score on a single A100 GPU in \cref{tab:time_alignment_score} as supporting evidence.
\end{remark}

\begin{table}[ht]
\centering
\caption{Computational cost is negligible. Time on a single A100 GPU.}
\label{tab:cost}
\begin{tabular}{l c}
\toprule
\textbf{Number of Domains} & \textbf{Time (s)} \\
\midrule
10 & 0.07 \\
100 & 0.09 \\
1000 & 0.48 \\
10000 & 21.85 \\
\bottomrule
\end{tabular}
\label{tab:time_alignment_score}
\end{table}

\subsection{Introducing latent variables}
\label{supp:subsec:latent_space}
We first give a lower bound to the objective \eqref{eq:primal:app} and introduce latent variables $\hs_i$, which will be used to couple the domains in the multi-modal setting.
Starting from the primal single-modal problem \eqref{eq:primal:app}, the following lower bound holds:
\begin{equation} \label{eq:rkm:app}
\begin{aligned}
J &= \frac{1}{2\lambda} \sum_{i=1}^k e_i^2 + \frac{1}{2} \norm{w}_\fro^2 \quad \text{s.t. } e_i = y_i - w^\top x_i, \, i=1,\ldots,k \\
&\geq \sum_{i=1}^k e_i \hs_i - \frac{\lambda}{2} \norm{\hs}_\fro^2 + \frac{1}{2} \norm{w}_\fro^2  \\
&= \sum_{i=1}^k (y_i - w^\top x_i) \hs_i - \frac{\lambda}{2} \norm{\hs}_\fro^2 + \frac{1}{2} \norm{w}_\fro^2 =: J_\text{SM},
\end{aligned}
\end{equation}
where $\lambda > 0$ is a regularization constants and $J_\text{SM}$ is the single modality objective.
The above bound is based on the property that for two arbitrary vectors $e, \hs$ one has
$\frac{1}{2\lambda}e^2+\frac{\lambda}{2}{\hs}^2 \geq e \hs, \, \forall e, \hs \in \mathbb{R}^k$.
The inequality can be verified using the Schur complement by writing in its quadratic form:
\begin{equation}
\frac{1}{2}
\begin{bmatrix}
  e^T & {\hs}^\top
\end{bmatrix}
\begin{bmatrix}
  \frac{1}{\lambda} I & I \\
  I & \lambda I
\end{bmatrix}
\begin{bmatrix}
  e \\
  \hs
\end{bmatrix}
\ge 0.
\notag
\end{equation}
From the Schur complement, it states the condition $\frac{1}{2} (\lambda I - I(\lambda I)I) \geq 0$, which proves the above inequality.
This is also known as conjugate feature duality \citep{suykens2017deep} or the Fenchel--Young inequality for quadratic functions \citep{rockafellar1974conjugate}.

Through the inequality, we have introduced latent variables, i.e. $\hs_i$, into the objective.
We proceed by studying the stationary condition of $J_\text{SM}$.
\begin{equation} \label{eq:rkm:stationary:app}
\left\{
\begin{aligned}
\frac{\partial J_\text{SM}}{\partial w} &= -\sum_{i=1}^k \hs_i x_i + w = 0 
  \quad \Rightarrow \quad w = \sum_{i=1}^k \hs_i x_i, \\[1ex]
\frac{\partial J_\text{SM}}{\partial \hs_i} &= y_i - w^\top x_i - \lambda \hs_i = 0 
  \quad \Rightarrow \quad \hs_i = \frac{1}{\lambda} \left( y_i - w^\top x_i \right) \quad \forall i.
\end{aligned}
\right.
\end{equation}
By eliminating $w$ in \eqref{eq:rkm:stationary:app}, we obtain
$$
w^\top x_i = \left( \sum_{j=1}^k \hs_j x_j \right)^\top x_i = \sum_{j=1}^k \hs_j\, (x_j^\top x_i) \quad \forall i.
$$
Thus the solution in the latent variables is
$$
\begin{aligned}
& \hs_i = \frac{1}{\lambda} \left( y_i - \sum_{j=1}^k \hs_j\, (x_j^\top x_i) \right) \\
& \hs = (K + \lambda I)^{-1} y.
\end{aligned}
$$
The score of domain $i$, i.e., \( S_i = w^\top x_i \), writes in terms of the latent variables as:
\begin{equation}
\begin{aligned}
S_i' = w^\top x_i 
= \left( \sum_{j=1}^k \hs_j x_j \right)^\top x_i = \sum_{j=1}^k \hs_j (x_j^\top x_i) 
= [K (K + \lambda I)^{-1} y]_i,
\end{aligned}
\end{equation}
which matches \eqref{eq:krls:app} obtained by the original problem \eqref{eq:primal:app}.
Let $y=1_k$, it recovers the uni-modal score in \Cref{subsec:alignment_score}.

\subsection{Proof of \Cref{prop:multi-modal}}
\label{supp:subsec:multi-modality}
We first characterize the stationary points of $J_\text{MM}$ defined in \cref{eq:mm_objective}, as the stationary conditions lead to the optimal solution in the dual of the multi-modal problem.
Note that the coupling across modalities can be achieved by creating a common latent space~\citep{houthuys2018multi,tao2024tensor}, i.e., by introducing the same latent variables $\alpha$ across all modalities in $J_\text{MM}$.
By taking the partial derivatives of the weights $w^{[v]}$ and the latent variables $\h$, the conditions of the stationary points leading to MM scores are characterized by:
\begin{equation}
\left\{
\begin{aligned}
\frac{\partial J_\text{MM}}{\partial w^{[v]}} &= -\sum_{i=1}^k \h_i x_i^{[v]} + w^{[v]} = 0 \quad \implies \quad w^{[v]} = \sum_{i=1}^k \h_i x_i^{[v]} \\
\frac{\partial J_\text{MM}}{\partial \h_i} &= \sum_{v=1}^V \Big( \delta_i^{[v]} - (w^{[v]})^\top x_i^{[v]} \Big) - \lambda \h_i = 0 \\
& \Rightarrow \sum_{v=1}^V \delta_i^{[v]} - \sum_{v=1}^V \Bigg( \sum_{j=1}^k \h_j \underbrace{(x_j^{[v]})^\top x_i^{[v]}}_{K_{ij}^{[v]}} \Bigg) - \lambda \h_i = 0\\
& \Rightarrow \sum_{v=1}^V \delta_i^{[v]}  - \sum_{j=1}^k \h_j \sum_{v=1}^V K_{ij}^{[v]} - \lambda \h_i = 0 \\
& \Rightarrow \sum_{v=1}^V \delta_i^{[v]}  - \sum_{j=1}^k \h_j K_{\mathrm{MM}_{ij}} - \lambda \h_i = 0, \text{where } K_{\mathrm{MM}_{ij}}^{[v]} = \sum_{v=1}^V K^{[v]}_{ij}.
\end{aligned}
\right.
\end{equation}

Define the multi-modal kernel matrix as $\K \in \mathbb{R}^{k \times k}$ with entries $\Ks_{\mathrm{MM}_{\scriptstyle ij}} = \sum_{v=1}^V \Ks_{ij}^{[v]}$, with $\Ks_{ij}^{[v]} =(x_i^{[v]})^\top x_j^{[v]}$.
The above conditions can be rewritten in matrix form as:
$$(\K + \lambda I) \h = \delta, $$
where  $\delta \in \R^{k}$ is the vector with entries $\delta_{i} = \sum_{v=1}^V \delta_i^{[v]}$ with $\delta_i^{[v]}\in\{0,1\}$ representing the existence of the modality $v$ of the domain $i$.
The solution in the latent variable therefore is
\[
\h = (\K + \lambda {I})^{-1} \delta.
\] 
We can compute the domain score for each modality as $S^{[v]}_i = {w^{[v]}}^\top x_i^{[v]}$.
For modality $v$, at optimality:
\[
w^{[v]} = \sum_{j=1}^k \h_j x_j^{[v]} \quad\Rightarrow\quad
S^{[v]}_i = {w^{[v]}}^\top x_i^{[v]}= \sum_{j=1}^k \h_j (x_j^{[v]})^\top x_i^{[v]}.
\]
Substituting $\h = (\K + \lambda I)^{-1} \delta. $, it yields in matrix form:
\[
S_i^{[v]} = \left[\Ks^{[v]} (\K+\lambda I)^{-1} \delta\right]_i,
\]
which is the multi-modal score of domain $i$ for modality $v$.
The ensemble score of domain $i$ then considers all modalities as $S_i = \sum_{v=1}^V S_i^{[v]}$.

\subsection{Eigendecomposition perspective}
\label{supp:subsec:spectral_soft}
Alignment score functions as a robust steering direction.
In fact, we can assert the following:

\begin{lemma}
Let $K_{MM} \in \mathbb{R}^{k \times k}$ be the multi-modal kernel matrix with SVD ${K}_{MM} = {U} {\Sigma} {U}^\top$, where $U=[u_1, \ldots, u_k]$ are the singular vectors and $\sigma_1 \geq \dots \geq \sigma_k$ are the singular values. 
The score $S_i$ derived from the objective in \cref{prop:multi-modal} is given by:
$S_i = \sum_{j=1}^k \left( \frac{\sigma_j}{\sigma_j + \lambda} \right) ({u}_j^\top \delta) ({u}_j)_i$.
\end{lemma}
\begin{proof}
By Proposition~\ref{prop:multi-modal}, the multi-modal score vector can be written as $S = K_{MM}(K_{MM}+\lambda I)^{-1}\delta$. Let $K_{MM} = U \Sigma U^\top$ be the eigendecomposition, where $U = [u_1,\dots,u_k]$ is orthonormal and $\Sigma = \mathrm{diag}(\sigma_1,\dots,\sigma_k)$. Then
$K_{MM}(K_{MM}+\lambda I)^{-1}
= U\Sigma U^\top U(\Sigma+\lambda I)^{-1}U^\top
= U\Sigma(\Sigma+\lambda I)^{-1}U^\top,$
where $\Sigma(\Sigma+\lambda I)^{-1}
= \mathrm{diag}\!\left(\tfrac{\sigma_1}{\sigma_1+\lambda},\dots,
\tfrac{\sigma_k}{\sigma_k+\lambda}\right)$. Writing
$\delta = \sum_{j=1}^k (u_j^\top\delta)u_j$ in the eigenbasis of $K_{MM}$, we obtain
$S
= \sum_{j=1}^k \frac{\sigma_j}{\sigma_j+\lambda}(u_j^\top\delta)\,u_j.$
Taking the $i$-th component yields
$S_i
= \sum_{j=1}^k \left(\frac{\sigma_j}{\sigma_j+\lambda}\right)
(u_j^\top\delta)\,(u_j)_i$.
\end{proof}

Therefore, our construction applies a spectral soft thresholding filter to the domain distribution.
The associated operator dampens the noisy directions (small eigenvalues) and effectively measures the projection of domain $i$ onto a robust semantic subspace (corresponding to large eigenvalues).

\section{Additional experiments} \label{supp:experiments}

\subsection{Difference between image and text modalities}
\label{supp:subsec:gram_matix}
In the multi-modality training process, there are two main challenges from the data perspective: \textit{(i)} domains may have different modalities, and \textit{(ii)} each domain's data features may vary significantly as captured by different modalities.

We take the LLaVA-OneVision dataset \citep{llava_onevision_data} as an example.
The LLaVA-OneVision dataset includes five domains along with two modalities, image and text. 
Four domains include both image and text modalities, while the ``Language'' domain only has text.
We visualize the embedding similarity matrix for text and image modalities independently in \cref{fig:gram_matrices}. 
It shows that the domain relationships in different modalities can vary considerably, although the kernel magnitudes remain comparable.
Note that the kernel matrices can be normalized to unit trace \citep{bach2004computing} if scales differed significantly.

\begin{figure}[h!]
    \centering
    \begin{subfigure}[t]{0.49\linewidth}
        \includegraphics[width=\linewidth]{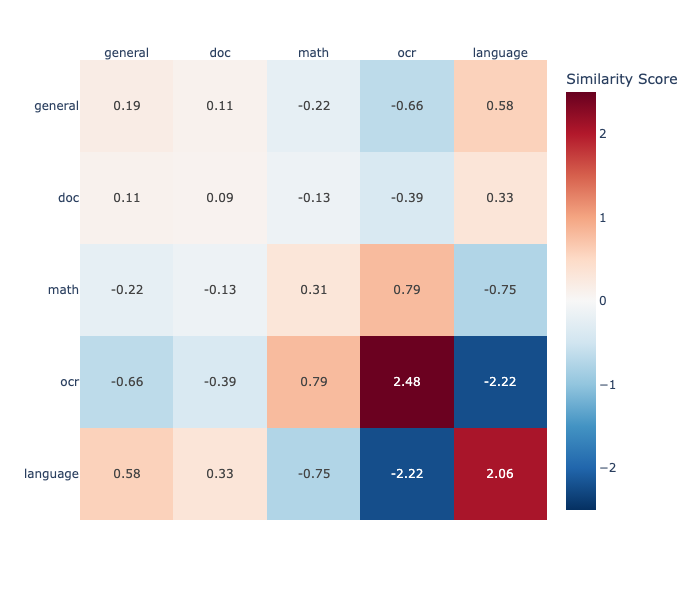}
        \caption{Text modality.}
        \label{fig:text_gram_matrix}
    \end{subfigure}%
    \hfill
    \begin{subfigure}[t]{0.49\linewidth}
        \includegraphics[width=\linewidth]{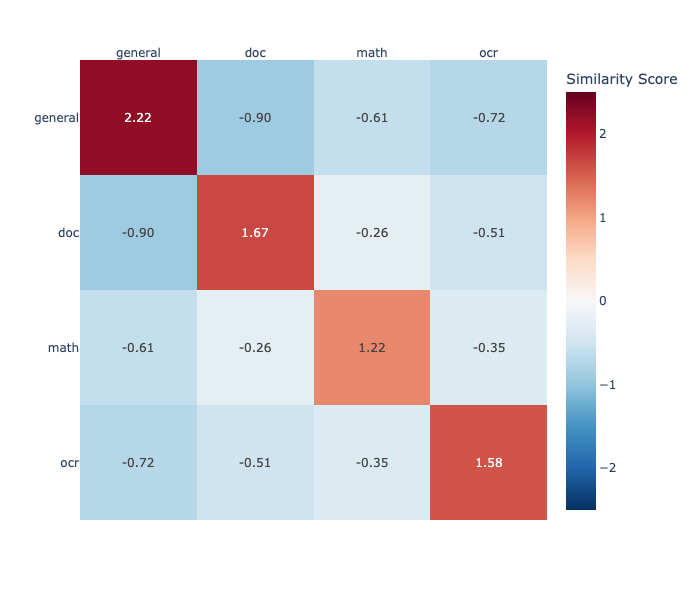}
        \caption{Image modality.}
        \label{fig:image_gram_matrix}
    \end{subfigure}
    \caption{Embedding kernel similarity matrix for different modalities.}
    \label{fig:gram_matrices}
\end{figure}

\subsection{Experimental setup} \label{supp:setups}
We use the LLaVA-OneVision publicly available data \citep{llava_onevision_data} for training and follow the domain segmentation in the LLaVA-OneVision paper \citep{li2024llavaonevisioneasyvisualtask}.
We use 3 seeds for the benchmark evaluations and report the standard deviations.

Note that some training datasets used in~\citep{li2024llavaonevisioneasyvisualtask} were not released and some datasets use different naming conventions than~\citep{li2024llavaonevisioneasyvisualtask}.
Our specific domain settings are: 

    $\bullet$ \textbf{General:} %
        aokvqa (cauldron,llava\_format) \citep{schwenk2022okvqa},
        clevr (cauldron,llava\_format) \citep{johnson2017clevr},
        hateful\_memes (cauldron,llava\_format) \citep{kiela2020hateful},
        image\_textualization (filtered) \citep{pi2024image},
        iconqa (cauldron,llava\_format) \citep{lu2021iconqa},
        IconQA (MathV360K) \citep{lu2021iconqa},
        scienceqa (cauldron,llava\_format) \citep{saikh2022scienceqa},
        scienceqa (nona\_context) \citep{saikh2022scienceqa},
        st\_vqa (cauldron,llava\_format) \citep{xia2023st},
        tallyqa (cauldron,llava\_format) \citep{acharya2019tallyqa},
        VisualWebInstruct (filtered) \citep{jia2025visualwebinstruct},
        visual7w (cauldron,llava\_format) \citep{zhu2016visual7w},
        vistext (cauldron) \citep{tang2023vistext},
        VizWiz (MathV360K) \citep{gurari2018vizwiz},
        vqarad (cauldron,llava\_format) \citep{lau2018dataset},
        vsr (cauldron,llava\_format) \citep{liu2023visual},
        websight (cauldron) \citep{laurençon2024unlocking},
        allava\_instruct\_laion4v \citep{chen2024allava},
        allava\_instruct\_vflan4v \citep{chen2024allava},
        vision\_flan (filtered) \citep{xu2024vision},
        intergps (cauldron,llava\_format) \citep{lu2021inter},
        llavar\_gpt4\_20k \citep{zhang2023llavar},
        sharegpt4o \citep{chen2024sharegpt4v},
        sharegpt4v (coco) \citep{chen2024sharegpt4v},
        sharegpt4v (knowledge) \citep{chen2024sharegpt4v},
        sharegpt4v (llava) \citep{chen2024sharegpt4v},
        sharegpt4v (sam) \citep{chen2024sharegpt4v} \\
    $\bullet$ \textbf{Doc/Chart/Screen:} %
        ai2d (cauldron,llava\_format) \citep{kembhavi2016ai2d},
        ai2d (gpt4v) \citep{kembhavi2016ai2d},
        ai2d (internvl) \citep{kembhavi2016ai2d},
        chart2text (cauldron) \citep{kantharaj2022chart},
        chartqa (cauldron,llava\_format) \citep{masry2022chartqa},
        diagram\_image\_to\_text (cauldron),
        dvqa (cauldron,llava\_format) \citep{kafle2018dvqa},
        figureqa (cauldron,llava\_format) \citep{kahou2017figureqa},
        hitab (cauldron,llava\_format) \citep{cheng2021hitab},
        infographic\_vqa \citep{mathew2022infographicvqa},
        infographic\_vqa\_llava\_format \citep{mathew2022infographicvqa},
        screen2words (cauldron) \citep{wang2021screen2words},
        tqa (cauldron,llava\_format) \citep{kembhavi2017you},
        ureader\_cap \citep{ye2023ureader},
        ureader\_ie \citep{ye2023ureader},
        robut\_sqa (cauldron) \citep{ghosh2024robusttabularqamodels},
        robut\_wikisql (cauldron) \citep{ghosh2024robusttabularqamodels},
        robut\_wtq (cauldron,llava\_format) \citep{ghosh2024robusttabularqamodels},
        visualmrc (cauldron) \citep{tanaka2021visualmrc},
        infographic (gpt4v) \citep{mathew2022infographicvqa},
        lrv\_chart \citep{liu2023aligning},
        mapqa (cauldron,llava\_format) \citep{chang2022mapqa},
        multihiertt (cauldron) \citep{zhao2022multihiertt}
        \\
    $\bullet$ \textbf{Math/Reasoning:}  %
        CLEVR-Math (MathV360K) \citep{lindstrom2022clevr},
        FigureQA (MathV360K) \citep{kahou2017figureqa},
        GEOS (MathV360K) \citep{seo2015solving},
        GeoQA+ (MathV360K) \citep{anand2024geovqa},
        Geometry3K (MathV360K) \citep{lu2021inter},
        MapQA (MathV360K) \citep{chang2022mapqa},
        Super-CLEVR (MathV360K) \citep{li2023super},
        TabMWP (MathV360K) \citep{lu2022dynamic},
        UniGeo (MathV360K) \citep{chen2022unigeo},
        geo170k (align) \citep{gao2023g},
        geo170k (qa) \citep{gao2023g},
        geomverse (cauldron) \citep{kazemi2023geomverse},
        mavis\_math\_metagen \citep{zhang2024mavis},
        mavis\_math\_rule\_geo \citep{zhang2024mavis},
        lrv\_normal (filtered) \citep{liu2023aligning},
        geo3k \citep{lu2021inter},
        raven (cauldron) \citep{zhang2019raven},
        PMC-VQA (MathV360K) \citep{zhang2023pmc},
        tabmwp (cauldron) \citep{lu2022dynamic}
        \\
    $\bullet$ \textbf{General OCR:}  %
        chrome\_writting \citep{WendlerGambot2023RenderedText},
        hme100k \citep{yuan2022syntax},
        iam (cauldron) \citep{marti2002iam},
        iiit5k \citep{mishra2012scene},
        k12\_printing,
        rendered\_text (cauldron) \citep{WendlerGambot2023RenderedText},
        textcaps \citep{sidorov2020textcaps},
        textocr (gpt4v) \citep{singh2021textocr},
        sroie,
        orand\_car\_a \\    
    $\bullet$ \textbf{Language:}  %
        magpie\_pro (l3\_80b\_mt),
        magpie\_pro (l3\_80b\_st),
        magpie\_pro (qwen2\_72b\_st) \citep{xu2024magpie} \\
    $\bullet$ \textbf{Video:} academic\_qa, youtube \citep{zhang2024videoinstructiontuningsynthetic}, ActivityNetQA \citep{yu2019activityqa}, NeXT-QA \cite{xiao2021next}, PerceptionTest \citep{patraucean2023perception}

\subsection{Embedding extraction and domain weight assignment}
\label{supp:subsec:embedding_extract}

For embedding computation, we use the pretrained LLaVA-OneVision model that has completed stage-1.5 pre-training and we randomly sample a subset of data from each domain. Given the presence of multiple datasets per domain, we extracted embeddings for 512 samples from each individual dataset. These sample embeddings were then averaged to create a single representation for each dataset. Subsequently, we averaged these dataset-level embeddings to capture the overall character of its respective domain.
Then, we use domain-level embeddings to compute domain weights.

Once we compute the domain weights $p_i$ using \cref{alg:mm-algo}, our training sampling strategy takes dataset size into account as follows. 
We sample datasets proportionally to their size within each domain, and then sample individual data points uniformly from the chosen dataset. This results in the final sampling probability for a dataset $DS$ in domain $D_i$ being $P = \frac{|DS|}{|D_i|}p_i$,
followed by uniformly sampling over instances in $DS$.

\subsection{Domain weights for the image-text instruction tuning (\cref{subsec:exp:text_image})}
\label{supp:subsec:domain_weights}
We report domain weights for \cref{subsec:exp:text_image} with five domains and two modalities in \cref{tab:domain_weights}.
Note that \textsc{avg} = $\frac{1}{2}$ (\textsc{Text}+\textsc{Image}).
\textsc{Image$^\dag$} sets its Language weight as same as \textsc{Human} and reweight the others in \textsc{Image}.

\begin{table}[ht]
    \centering
    \caption{\textbf{VLM Mixtures for the image-text instruction tuning.} Domain weights of different mixing strategies.
    \textsc{Image$^\dag$} sets its Language weight as same as \textsc{Human} and reweight the others in \textsc{Image}.
    }
    \begin{tabular}{lccccccc|c}
    \toprule
       Domain & \textsc{Uni.} & \textsc{Human} & \textsc{Text} & \textsc{Image} & \textsc{Avg} & \textsc{Fused} & \ours & \textsc{Image$^\dag$} \\
    \midrule
       General & 20.00 & 36.10 & 20.90 & 35.66 & 28.28 & 14.74 & 22.09 & 30.56 \\
       Doc/Chart/Screen & 20.00 & 20.60  & 43.28 & 29.49 & 36.29 & 40.95 & 31.86 & 25.27\\
       Math/Reasoning & 20.00 & 20.10 & 15.24 & 17.92 & 16.58 & 20.21 & 16.63 & 15.36 \\
       General OCR & 20.00 & 8.90  & 10.22 & 16.93 & 13.58 & 14.14 & 15.66 & 14.51 \\
       Language & 20.00 & 14.30  & 10.35 & 0.00 & 5.18 & 9.95 & 13.76 & 14.30 \\
    \bottomrule
    \end{tabular}
    \label{tab:domain_weights}
\end{table}

\subsection{Performance of domain weights computed by single modality}
We add two unimodal strategies, \textsc{Text} and \textsc{Image$^\dag$}, in \cref{tab:eval_si_0.5B_supp,tab:eval_si_7B_supp} as addition for \cref{tab:eval_si_0.5B,tab:eval_si_7B}.
These two unimodal methods compute the domain weights derived from single modality in \Cref{subsec:alignment_score}, based solely on text or image embeddings. 
Note that the Language domain does not have image data, thus \textsc{Image} has 0\% on this domain.
For a more reasonable comparison, we set its Language domain weight to the same as \textsc{Human} and reweight the others, finalizing to \textsc{Image$^\dag$} in \cref{tab:domain_weights}.
Importantly, \ours still outperforms these unimodal strategies.

\label{supp:subsec:acc_single_modality}
\begin{table}[ht]
    \centering
    \caption{\textbf{Comparison of data mixing strategies for LLaVA-0.5B image-text instruction tuning.}
    Results are reported as 0-shot accuracy across ten evaluation benchmarks.
    \ours achieves the best average performance, including the single-modality methods.
    }
    \begin{tabular}{lccccc|cc}
    \toprule 
    Benchmark & \textsc{Uniform} & \textsc{Human} & \textsc{Avg} & \textsc{fused} & \ours & \textsc{Text} & \textsc{Image$^\dag$} \\
    \midrule
        AI2D        & 42.78 & 43.75 & 45.50 & 44.59 & 43.52 & 45.95 & 44.33 \\
        DocVQA      & 42.90 & 42.66 & 42.44 & 42.67 & 42.92 & 43.08 & 42.42 \\
        InfoVQA     & 22.25 & 22.61 & 22.43 & 23.50 & 22.13 & 23.45 & 21.47 \\
        MathVerse   & 18.27 & 17.26 & 18.32 & 19.29 & 18.91 & 16.50 & 18.53 \\
        MMBench     & 36.34 & 40.21 & 39.86 & 37.71 & 42.44 & 35.82 & 39.00 \\
        MMStar      & 33.45 & 36.04 & 33.50 & 34.44 & 35.88 & 34.19 & 34.67 \\
        MMMU        & 30.00 & 29.67 & 29.00 & 29.22 & 29.78 & 27.89 & 30.67 \\
        ScienceQA   & 62.42 & 65.84 & 64.80 & 63.46 & 64.50 & 64.60 & 63.86 \\
        OCRBench    & 45.30 & 44.60 & 45.30 & 43.50 & 45.80 & 45.30 & 45.20 \\
        RealworldQA & 46.27 & 44.05 & 45.49 & 45.36 & 46.54 & 45.36 & 46.67 \\
    \midrule
        \rowcolor{gray!15} Average   & 38.00 & 38.67 & 38.66 & 38.37 & \textbf{39.24} & 38.21 & 38.69 \\
        \rowcolor{gray!15} \# over \textsc{Uniform} & - & 5/10 & 6/10 & 6/10 & 8/10 & 5/10 & 7/10  \\
    \bottomrule
    \end{tabular}    
    \label{tab:eval_si_0.5B_supp}
\end{table}

\begin{table}[!ht]
    \centering
    \caption{\textbf{Comparison of data mixing strategies for LLaVA-7B image-text instruction tuning.}
    Results are reported as 0-shot accuracy across ten evaluation benchmarks.
    \ours achieves the best average performance, including single-modality methods.
    }
    \begin{tabular}{lccccc|cc}
    \toprule
    Benchmark & \textsc{Uniform} & \textsc{Human} & \textsc{Avg} & \textsc{Fused} & \ours & \textsc{Text} & \textsc{Image$^\dag$} \\
    \midrule
        AI2D        & 74.48 & 74.03 & 75.10 & 75.74 & 75.58 & 75.42 & 75.58 \\
        DocVQA      & 57.91 & 58.64 & 58.28 & 57.29 & 58.32 & 58.73 & 57.86 \\
        InfoVQA     & 34.76 & 35.91 & 36.95 & 36.06 & 36.23 & 36.83 & 36.22 \\
        MathVerse   & 29.31 & 26.85 & 27.33 & 28.68 & 28.55 & 26.14 & 27.83 \\
        MMBench     & 75.69 & 76.12 & 76.23 & 75.77 & 75.74 & 75.60 & 76.98 \\
        MMStar      & 49.04 & 50.26 & 50.44 & 49.46 & 50.19 & 49.51 & 50.72 \\
        MMMU        & 46.33 & 46.78 & 46.78 & 46.78 & 46.89 & 47.11 & 45.67 \\
        ScienceQA   & 87.31 & 90.38 & 89.53 & 85.52 & 90.23 & 86.91 & 90.08 \\
        OCRBench    & 56.80 & 57.30 & 56.70 & 56.60 & 57.90 & 57.70 & 57.50 \\
        RealworldQA & 58.17 & 57.91 & 56.99 & 57.65 & 57.47 & 57.65 & 57.39 \\
    \midrule
        \rowcolor{gray!15} Average   & 56.98 & 57.42 & 57.43 & 56.96 & \textbf{57.71} & 57.16 & 57.59 \\
        \rowcolor{gray!15} \# over \textsc{Uniform} & - & 7/10 & 7/10 & 5/10 & 8/10 & 6/10 & 6/10  \\
    \bottomrule
    \end{tabular}    
    \label{tab:eval_si_7B_supp}
\end{table}

\subsection{Ablation studies} \label{supp:ablations}

\paragraph{Regularization parameter $\lambda$.}
The parameter $\lambda$ is related to the degree of regularization. 
Despite this control, 
\cref{tab:domain_weights_sweep} demonstrates that our obtained domain weights are largely stable with respect to changes in $\lambda$.

\paragraph{Number of samples for embedding extraction.}
As we discussed in \cref{supp:subsec:embedding_extract}, we sample a subset of datasets for embedding extraction. We test the robustness of domain weights with respect to the number of samples. 
The domain weights based on 256, 512, or 1024 samples from each individual dataset are reported in \cref{tab:domain_weights_sweep}, which confirms that the domain weights obtained are stable regardless of the number of samples.

\paragraph{Embedding aggregation.}
Except for averaging the dataset-level averaged embeddings to represent each domain, 
another way is to aggregate dataset-level embeddings to domain embeddings according to their dataset sizes.
Basically, sum the dataset-level embeddings reweighted by their sizes as domain weights.
The domain weights computed by these two strategies are highly similar, as reported in \cref{tab:domain_weights_sweep}.

\paragraph{Model sizes.}
We use the domain weights obtained from the LLaVA-0.5B model's embeddings and test the transferability for LLaVA-7B shown in \cref{tab:eval_si_7B}.
We compute domain weights based on the embeddings from LLaVA-7B as well. The result reported in \cref{tab:domain_weights_sweep} demonstrates that domain weights are stable across model sizes.

\paragraph{The number of training steps for the pretrained model.}
Regarding the sufficiency of feature extraction, the latent representations are from the mid-stage checkpoint of LLaVA-OneVision \citep{li2024llavaonevisioneasyvisualtask}, which is a well-trained model.
We additionally conducted a stability analysis by continuing training this checkpoint on its public mid-training data for 500 and 1000 additional steps and recomputed domain weights using our method.
The results, presented in \cref{tab:domain_weights_sweep}, demonstrate that the domain weights remain stable throughout training.

\paragraph{Number of domains.} 
We run additional experiments with a reduced number of domains. We exclude `General' from the original five domains, and the new domain weights obtained by \ours are: 26.7\% Doc/Chart/Screen, 28.7\% Math/Reasoning, 31.6\% General OCR, and 13.0\% Language. The domain weights of \textsc{Uniform} are 25\% per domain. \cref{tab:eval_si_4domain} demonstrates that \ours consistently shows a higher average accuracy. 
This validates the robustness of our method across different numbers of domains.
Furthermore, the experiment in \cref{subsec:exp:add_video}, which introduces a Video domain, demonstrates that \ours remains effective as the number of domains changes.

\begin{table*}[ht]
    \centering
    \caption{\textbf{Domain weights across $\lambda$ values, number of samples, model sizes, and the number of steps for pretrained checkpoints.}
    Our method is robust to the choice of $\lambda$, the number of samples used, embedding aggregation methods, model sizes and the number of steps.
    }
    \begin{tabular}{l | ccc | ccc | cc }
        \toprule
        \multirow{2}{*}{Domain}  & \multicolumn{3}{c|}{\textbf{$\lambda$ Values}} & \multicolumn{3}{c|}{\textbf{Number of Samples}} & \multicolumn{2}{c}{\textbf{Aggregate embeddings}} \\
         & 1 & 10 & 100 & 256 & 512 & 1024  & Equally & Dataset sizes\\
        \midrule
       General & 21.57 & 22.09 & 23.14 & 24.48  & 22.09 & 23.96 & 22.09 & 23.20 \\
       Doc/Chart/Screen & 28.90 & 31.86 & 34.79 & 27.32 & 31.86 & 30.84 & 31.86 & 30.98 \\
       Math/Reasoning & 17.47 & 16.63 & 15.95 & 18.71  & 16.63 & 17.32 & 16.63 & 17.80 \\
       General OCR & 16.74 & 15.66 & 14.49 & 16.72 & 15.66 & 17.04 & 15.66 & 16.91 \\
       Language & 15.32 & 13.76 & 11.64 & 12.77 &  13.76 & 10.84 & 13.76 & 11.11 \\
        \bottomrule
    \end{tabular}
    \begin{tabular}{l | cc | ccc }
        \toprule
       \multirow{2}{*}{Domain}  & \multicolumn{2}{c|}{\textbf{Model sizes}} & \multicolumn{3}{c}{\textbf{\# training steps for pretrained model}} \\
        & 0.5B & 7B & Pretrained model & +500 steps & +1000 steps  \\
        \midrule
       General & 22.09 & 21.15 & 22.09 & 21.65 & 22.59  \\
       Doc/Chart/Screen & 31.86 & 29.91 & 31.86 & 30.12 & 29.96 \\
       Math/Reasoning & 16.63 & 19.64 & 16.63 & 18.74 & 18.82 \\
       General OCR & 15.66 & 17.06 & 15.66 & 15.17 & 15.82  \\
       Language & 13.76 & 12.23 & 13.76 & 14.33 & 12.99  \\
        \bottomrule
    \end{tabular}
    \label{tab:domain_weights_sweep}
\end{table*}

\begin{table}[ht]
\centering
\caption{\textbf{Comparison of data mixtures on 4 domains for LLaVA-0.5B image-text instruction tuning.} \ours is robust across different numbers of domains.}
\label{tab:eval_si_4domain}
\begin{tabular}{lcc}
\toprule
\textbf{Benchmark} & \textsc{Uniform} & \ours \\
\midrule
AI2D & 42.75 & 43.75 \\
DocVQA & 40.79 & 41.71 \\
InfoVQA & 22.89 & 22.96 \\
MathVerse & 17.51 & 18.27 \\
MMBench & 30.76 & 33.59 \\
MMStar & 35.68 & 33.24 \\
MMMU & 28.89 & 31.78 \\
ScienceQA & 53.99 & 54.09 \\
OCRBench & 43.30 & 44.70 \\
RealworldQA & 38.30 & 42.88 \\
\midrule
\rowcolor{gray!15} Average & 35.48 & \textbf{36.69} \\
\rowcolor{gray!15} Number over \textsc{Uniform} & - & 9/10 \\
\bottomrule
\end{tabular}
\end{table}

\subsection{Alignment score vs Orthogonal score}
\label{supp:subsec:orthogonal_score}
To further ablate the benefit introduced in \cref{sec:method}, we introduce the ``Orthogonal Score'' to quantify the uniqueness of different domains, i.e., downweighting high-alignment domains.
Specifically for the domain $j$, set $e_i = \gamma_i-w^\top x_i$, with $\gamma_i=1$ if $i=j$, $0$ otherwise in \cref{eq:primal}.
And \cref{eq:mm_objective} becomes 
$$J_{\text{MM}}^{\text{orth}} = \sum_{v=1}^V \sum_{i=1}^k \Big[ (\delta_i^{[v]} \gamma_i - (w^{[v]})^\top x_i^{[v]}) \alpha_i - \frac{\lambda}{2} \alpha_i^2 \Big] + \frac{1}{2} \sum_{v=1}^V \|w^{[v]}\|_\mathrm{F}^2.$$
The orthogonal score derived by this new objective is $$S_j^{[v]} = \delta_j \left[K^{[v]} (K_{\text{MM}}+\lambda I)^{-1}\right]_{jj}.$$

We report the downstream performance of domain weights computed through Orthogonal score in \Cref{tab:eval_orthogonal_score}.
For comparison, we also present the original scores used in our method (\ours) below. 
The performance of \ours shows higher average accuracy than Orthogonal score.
Importantly, Orthogonal score performs worse than UNIFORM on average (37.62 vs 38.00).

\begin{table}[ht]
\centering
\caption{\textbf{Comparison of Orthogonal and Alignment scores for LLaVA-0.5B image-text instruction tuning.} Orthogonal score is even worse than UNIFORM (37.62 vs 38.00).}
\label{tab:eval_orthogonal_score}
\begin{tabular}{lcc}
\toprule
\textbf{Benchmark} & Orthogonal score & Alignment score (\ours) \\
\midrule
AI2D & 43.04 & 43.52 \\
DocVQA & 41.58 & 42.92 \\
InfoVQA & 21.49 & 22.13 \\
MathVerse & 15.99 & 18.91 \\
MMBench & 32.65 & 42.44 \\
MMStar & 34.52 & 35.88 \\
MMMU & 30.22 & 29.78 \\
ScienceQA & 64.25 & 64.50 \\
OCRBench & 46.30 & 45.80 \\
RealworldQA & 46.14 & 46.54 \\
\midrule
\rowcolor{gray!15} Average & 37.62 & \textbf{39.24} \\
\bottomrule
\end{tabular}
\end{table}

\subsection{Domain weights for video-image-text instruction tuning (\cref{subsec:exp:add_video})}
\label{supp:subsec:domain_weights_video}
We report domain weights for \cref{subsec:exp:add_video} with six domains and three modalities in \cref{tab:domain_weights_video}.
Note that \textsc{avg} = $\frac{1}{3}$ (\textsc{Text}+\textsc{Image}+\textsc{Video}).
\begin{table}[!ht]
    \centering
    \caption{\textbf{VLM Mixtures.} Domain weights across different mixing strategies for three modalities.
    }
    \begin{tabular}{lccccccc}
    \toprule
       Domain & \textsc{Uniform} & \textsc{Text} & \textsc{Image} & \textsc{Video} & \textsc{Avg} & \textsc{Fused} & \ours \\
    \midrule
       General & 16.67 & 16.62 & 35.66  & 0.00 & 17.43 & 10.77 & 24.66 \\
       Doc/Chart/Screen & 16.67 & 14.42 & 29.49 & 0.00 & 16.70 & 13.20 & 14.74 \\
       Math/Reasoning & 16.67 & 30.73 & 17.92 & 0.00 & 16.22 & 9.44 & 16.65 \\
       General OCR & 16.67 & 9.27 & 16.93 & 0.00 & 8.73 & 38.60 & 17.55 \\
       Language & 16.67 & 13.89 & 0.00 & 0.00 & 4.63 & 16.73 & 5.91 \\
       Video & 16.67 & 15.08 & 0.00 & 100.00 & 38.36 & 11.26 & 20.49 \\
    \bottomrule
    \end{tabular}
    \label{tab:domain_weights_video}
\end{table}

\subsection{\cref{tab:eval_video} with standard deviations}
\label{supp:subsec:acc_video_std}
We show the results with standard deviations of \cref{tab:eval_video} in \cref{tab:eval_video_0.5B,tab:eval_video_7B}.

\begin{table}[!ht]
    \centering
    \caption{\textbf{Comparison of data mixtures for LLaVA-0.5B video-image-text instruction tuning.}}
    \begin{tabularx}{\textwidth}{lCCCC}
    \toprule
    Benchmark & \textsc{Uniform} & \textsc{Avg} & \textsc{Fused} & \ours \\
    \midrule
        AI2D        & $41.68_{\pm 0.08}$ & $42.81_{\pm 0.09}$ & $42.84_{\pm 0.10}$ & $42.88_{\pm 0.04}$ \\
        DocVQA      & $42.20_{\pm 0.06}$ & $41.68_{\pm 0.05}$ & $41.29_{\pm 0.06}$ & $42.54_{\pm 0.08}$ \\
        InfoVQA     & $21.65_{\pm 0.07}$ & $21.97_{\pm 0.06}$ & $21.17_{\pm 0.08}$ & $22.40_{\pm 0.10}$ \\
        MathVerse   & $15.61_{\pm 0.10}$ & $15.62_{\pm 0.14}$ & $17.77_{\pm 0.11}$ & $15.10_{\pm 0.08}$ \\
        MMBench     & $34.36_{\pm 0.02}$ & $26.80_{\pm 0.03}$ & $35.14_{\pm 0.06}$ & $34.45_{\pm 0.04}$ \\
        MMStar      & $30.43_{\pm 0.05}$ & $35.54_{\pm 0.08}$ & $36.14_{\pm 0.06}$ & $33.97_{\pm 0.04}$ \\
        MMMU        & $30.00_{\pm 0.15}$ & $29.78_{\pm 0.09}$ & $30.44_{\pm 0.13}$ & $29.78_{\pm 0.11}$ \\
        ScienceQA   & $60.29_{\pm 0.11}$ & $60.29_{\pm 0.10}$ & $59.40_{\pm 0.12}$ & $61.03_{\pm 0.09}$ \\
        OCRBench    & $45.30_{\pm 0.12}$ & $43.20_{\pm 0.07}$ & $46.60_{\pm 0.08}$ & $45.00_{\pm 0.15}$ \\
        RealworldQA & $47.19_{\pm 0.18}$ & $46.41_{\pm 0.16}$ & $46.27_{\pm 0.12}$ & $47.32_{\pm 0.10}$ \\
        Video-MMMU  & $13.78_{\pm 0.08}$ & $13.78_{\pm 0.04}$ & $12.78_{\pm 0.10}$ & $13.84_{\pm 0.06}$ \\
        MVBench     & $36.67_{\pm 0.06}$ & $36.50_{\pm 0.10}$ & $37.02_{\pm 0.10}$ & $40.70_{\pm 0.12}$ \\
    \midrule
        \rowcolor{gray!15} Average      & $34.93_{\pm 0.10}$ & $34.53_{\pm 0.09}$ & $34.74_{\pm 0.10}$ & $\mathbf{35.75}_{\pm 0.09}$ \\
        \rowcolor{gray!15} Number over \textsc{Uniform} & - & 4/12 & 7/12 & 9/12 \\
    \bottomrule
    \end{tabularx}
    \label{tab:eval_video_0.5B}
\end{table}

\begin{table}[!ht]
    \centering
    \caption{\textbf{Comparison of data mixtures for LLaVA-7B video-image-text instruction tuning.}}
    \begin{tabularx}{\textwidth}{lCCCC}
    \toprule
    Benchmark & \textsc{Uniform} & \textsc{Avg} & \textsc{Fused} & \ours \\
    \midrule
        AI2D        & $71.83_{\pm 0.03}$ & $72.41_{\pm 0.08}$ & $72.83_{\pm 0.06}$ & $72.15_{\pm 0.06}$ \\
        DocVQA      & $56.47_{\pm 0.04}$ & $56.42_{\pm 0.06}$ & $55.67_{\pm 0.08}$ & $57.51_{\pm 0.10}$ \\
        InfoVQA     & $35.74_{\pm 0.12}$ & $34.65_{\pm 0.10}$ & $34.40_{\pm 0.07}$ & $35.89_{\pm 0.06}$ \\
        MathVerse   & $25.63_{\pm 0.11}$ & $25.52_{\pm 0.12}$ & $24.75_{\pm 0.08}$ & $26.40_{\pm 0.14}$ \\
        MMBench     & $71.05_{\pm 0.03}$ & $75.52_{\pm 0.06}$ & $73.28_{\pm 0.04}$ & $74.57_{\pm 0.05}$ \\
        MMStar      & $48.18_{\pm 0.08}$ & $49.03_{\pm 0.06}$ & $46.55_{\pm 0.04}$ & $48.79_{\pm 0.07}$ \\
        MMMU        & $45.67_{\pm 0.14}$ & $45.11_{\pm 0.10}$ & $44.78_{\pm 0.12}$ & $45.56_{\pm 0.13}$ \\
        ScienceQA   & $83.44_{\pm 0.04}$ & $86.07_{\pm 0.13}$ & $83.29_{\pm 0.10}$ & $87.26_{\pm 0.08}$ \\
        OCRBench    & $56.50_{\pm 0.11}$ & $56.90_{\pm 0.08}$ & $57.60_{\pm 0.09}$ & $57.20_{\pm 0.14}$ \\
        RealworldQA & $57.91_{\pm 0.16}$ & $56.99_{\pm 0.15}$ & $59.22_{\pm 0.08}$ & $57.39_{\pm 0.09}$ \\
        Video-MMMU  & $29.78_{\pm 0.07}$ & $30.56_{\pm 0.05}$ & $29.11_{\pm 0.08}$ & $30.33_{\pm 0.06}$ \\
        MVBench     & $52.73_{\pm 0.08}$ & $51.58_{\pm 0.12}$ & $53.12_{\pm 0.07}$ & $53.60_{\pm 0.11}$ \\
    \midrule
        \rowcolor{gray!15} Average  & $52.91_{\pm 0.09}$ & $53.39_{\pm 0.10}$ & $52.88_{\pm 0.08}$ & $\mathbf{54.40}_{\pm 0.10}$ \\
        \rowcolor{gray!15} Number over \textsc{Uniform} & - & 6/12 & 5/12 & 10/12 \\
    \bottomrule
    \end{tabularx}
    \label{tab:eval_video_7B}
\end{table}

\subsection{Domain weights transfer to Qwen2-VL}
\label{supp:subset:qwen-vl}
To explore the generality of \ours across different architectures, we test domain weights obtained from LLaVA-0.5B on Qwen-VL-2B \citep{Qwen2-VL} in \cref{tab:qwen2_vl} with the same setup of \cref{subsec:exp:add_video}.
The results show that \ours maintains its benefit compared with other baselines also on Qwen-VL-2B.
It would be interesting to further explore the benefit of new domain weights on more model architectures.

\begin{table}[ht]
\centering
\caption{\textbf{Transfer domain weights from LLaVA-0.5B to Qwen2-VL-2B for video-image-text instruction tuning.}}
\begin{tabular}{lcccc}
\toprule
    Benchmark & \textsc{Uniform} & \textsc{Avg} & \textsc{Fused} & \ours \\
\midrule
    AI2D & 67.78 & 67.94 & 67.29 & 68.26 \\
    DocVQA & 75.10 & 75.76 & 80.11 & 78.11 \\
    InfoVQA & 42.69 & 42.42 & 42.72 & 44.02 \\
    MathVerse & 21.19 & 18.78 & 23.73 & 23.98 \\
    MMBench & 59.02 & 56.44 & 55.33 & 59.71 \\
    MMStar & 41.37 & 42.90 & 40.89 & 41.11 \\
    MMMU & 37.44 & 36.00 & 35.98 & 37.44 \\
    ScienceQA & 77.89 & 79.13 & 77.84 & 79.23 \\
    OCRBench & 71.60 & 73.80 & 72.90 & 72.30 \\
    RealworldQA & 58.82 & 58.82 & 58.04 & 58.69 \\
    Video\_MMMU & 21.02 & 21.50 & 20.32 & 20.83 \\
    MVBench & 56.38 & 56.88 & 56.88 & 56.92 \\
\midrule
\rowcolor{gray!15} Average& 52.53 & 52.53 & 52.67 & \textbf{53.38} \\
\rowcolor{gray!15} Number over \textsc{Uniform} & - & 7/12 & 6/12 & 8/12 \\
\bottomrule
\end{tabular}
\label{tab:qwen2_vl}
\end{table}

\section{Further comparisons with related works}

\paragraph{Data mixing in LMs.}
Finding a high-quality data composition for LM pretraining is crucial for improved performance.
Domain reweighting improves LM downstream performance by rebalancing data contributions from different sources \citep{brown2020languagemodelsfewshotlearners,touvron2023llama,blakeney2024does}, but manual data mixing is not scalable and may lead to suboptimal domain weights \citep{albalak2024surveydataselectionlanguage,jiang2024adaptive,aryabumi2024code}.
Therefore, some works in the LM field explore the data mixing problems.
DoReMi \citep{xie_doremi_2023} employs a small proxy model to redistribute weights across various domains using Group DRO \citep{sagawa2019distributionally}, thereby enhancing the training effectiveness of large base models. 
Group DRO was also used in \citep{thudi2025mixmax}. 
DoGE~\citep{fan_doge_2024,fan2024dynamic} employs approximate bilevel optimization to train proxy models for domain weight determination.
Recently, \citep{liu2024regmix} employs linear regression models to approximate validation loss across diverse data mixtures by training a large number of very small proxy models.
\citet{chen2024aioli} create a more general framework with the above methods as specific instantiations.
Nevertheless, proxy-based methods necessitate algorithmic modifications in the training procedure, incurring supplementary proxy computational expenditure when multiple training stages are required, as is the case in VLMs.
Moreover, these approaches are limited to small proxy models, which may not be feasible within the context of VLMs with both vision and language models.
Other approaches focus on optimizing certain skills, e.g., \citet{chen2023skill} introduced a skills-oriented framework for modulation of data mixtures during model training.
\citet{thudi2025mixmin} use proxy models in a bilevel optimization framework to optimize the data mixture with downstream data samples.
\citet{held2025optimizing} propose mixing by estimating influence on downstream performance from each domain and assuming a linear model for the mixture weights.
Another line of works featurize the datasets by deriving a compact domain representation, e.g., through clustering~\cite{zhang2025domainvec} or pooling~\cite{xie2025chameleon}. The domain featureizations are then used to optimize dataset compositions, i.e., deciding weights to assign to the components of a combined dataset, through, e.g., correlation with validation set performance~\citep{zhang2025domainvec} or through leverage scores~\citep{xie2025chameleon}.
Drawing inspiration from scaling law research \citep{kaplan2020scaling,hoffmann2022training}, Data Mixing Laws~\citep{ye2024data} characterize the relationship between mixtures through exponential formulations, with other data mixture scaling laws proposed in \citep{que2024d,gu2024cmr,jiang2024adaptive,kang2024autoscale}.
Overall, these works have shown that choosing the right data mixture in LMs can boost performance significantly in terms of perplexity and downstream tasks' accuracy.
However, these works are limited to LMs and do not consider the challenges posed by VLMs, which require a more complex data mixture strategy due to, e.g., the multimodal nature of the data, missing modalities, and different training pipelines.

\paragraph{Leverage score mixing.}
Our unimodal scores measure the contribution of each domain w.r.t. the weight vector $w$ optimally fitted on the entire distribution, formulated through the regression task \eqref{eq:primal} with uniform target values across domains. 
Other common related learning tasks include ridge leverage scores (RLS). 
RLS measures the uniqueness of each data point through a weighted norm of the rows of the eigenvector matrix of the covariance. 
Specifically, RLS aims to find a vector $w$ being orthogonal to all data points except $x_i$. 
It can be formulated as regression for each domain $i$ separately with error variables $e_j = \gamma_j - w^\top x_j$, with
$
\gamma_j = 1 \text{ if } j=i,\; 0 \text{ otherwise}
$
.
\cite{xie2025chameleon} assign higher weights to domains with lower RLS, thus employing inverse RLS as a proxy for dominant directions. 
In our scores, we formulate \eqref{eq:primal} that seeks $w$ capturing shared structure across the \emph{entire collection} of domain embeddings. 
This is achieved by assigning a uniform target value for all domains, i.e., $e_i = 1 - w^\top x_i$. Our scores thus have a different objective, which can be analyzed in the following perspectives.
\emph{(i)} Our resulting score directly quantifies domain relevance, allowing for direct reweighting without inversion. This foundational difference in objective allows for a more natural and direct measure of how much each domain contributes to the shared structure in the embeddings.  
\emph{(ii)} Our work aims to capture multi-modal couplings in the data domains. Our new direct formulation \eqref{eq:primal} facilitates the construction of the multi-modal objective. 
By introducing shared latent variables $\alpha_i$ via the Fenchel-Young inequality~\eqref{eq:rkm}, we achieve principled coupling across multiple modalities in the dual formulation, whereas \cite{xie2025chameleon} cannot easily achieve such extension through the inverse RLS.

\paragraph{Data strategies for VLMs.}
Data mixtures in VLMs are typically hand-picked by the model developers based on intuition or large grid searches, and no systematic approach is used to select the training data mixture.
{Qwen-VL}~\citep{bai2023qwenvlversatilevisionlanguagemodel} employs a three-stage training pipeline utilizing a multilingual and multimodal corpus. The pre-training data is task-specific, e.g., captioning and OCR data.
In the instruction tuning stage, they combine multi-modal and text-only dialogue to mantain language capabilities performance.
LLaVA \citep{liu2023visualinstructiontuning,li2024llavaonevisioneasyvisualtask,liu2024improvedbaselinesvisualinstruction} additionally integrates LLM-generated instruction-following data with visual inputs.
They openly release the LLaVA-OneVision~\citep{li2024llavaonevisioneasyvisualtask} datasets as collections of domain-specific data, which we use in our experiments.
{Bunny}~\citep{he2024efficientmultimodallearningdatacentric} emphasizes the importance of high-quality data curation. Their approach focuses on finding coresets of the training dataset to improve model performance by removing uninformative image-text pairs.
{SAIL-VL}~\citep{dong2025scalablevisionlanguagemodel} constructs a high-quality dataset through recaptioning via existing frontier VLMs. This curated dataset facilitates effective pretraining and fine-tuning of VLMs across various scales.
Previous data selection works on CLIP training include, e.g., {CiT}~\citep{xu2023citcurationtrainingeffective}, which proposes a dynamic data curation method coupling a data objective into the learning process by measuring the similarity between text embeddings and task-specific metadata; and, {SIEVE}~\citep{mahmoud2024sievemultimodaldatasetpruning}, which introduces a dataset pruning technique using synthetic captions generated by image-captioning models, allowing to identify and remove noisy or misaligned samples, enhancing dataset quality.
Data strategies for VLMs have also been studied, e.g., data cleaning, toxicity removal, deduplication; see \citep{bai2024surveymultimodallargelanguage} for a comprehensive survey.
{DataComp}~\citep{gadre2023datacompsearchgenerationmultimodal} deals with data filtering.
{Infinity-MM}~\citep{gu2025infinitymmscalingmultimodalperformance} investigates the scaling of multimodal models by increasing both model capacity and training data volume. 
W.r.t. integrating multiple modalities more in general, this is a long-standing challenge in machine learning~\citep{baltruvsaitis2018multimodal,huang2023multimodal,li2024survey}. 
Simple fusion methods, such as early fusion via concatenation \citep{barnum2020benefitsearlyfusionmultimodal} or late fusion by ensembling \citep{boulahia2021early, li2024multimodalalignmentfusionsurvey}, are often used.
Another strategy is to learn a shared latent space where modalities are mapped to, enabling tasks like cross-modal retrieval~\citep{liu2023visualinstructiontuning,zhu2023minigpt}, using contrastive learning~\citep{radford2021learning,alayrac2022flamingo} or duality~\citep{houthuys2018multi,tao2024tensor}.
Other methods utilize attention to represent interaction between modality-specific encoders \citep{lu2019vilbert,cai2024internlm2}.
Overall, the composition of training data is crucial for the performance of VLMs.
To avoid reliance on expensive iterative performance measurements, our work introduces a method that can automatically assign appropriate resampling weights to each multi-modal domain of VLM training data.